\documentclass[12pt]{colt2024} % Anonymized submission

\usepackage{microtype}
\usepackage{graphicx}
\usepackage{algorithmic}
\usepackage{algorithm}
\newtheorem{assumption}[theorem]{Assumption}
\usepackage{booktabs} % for professional tables
\usepackage{tikz}
\usetikzlibrary{positioning,arrows,decorations,decorations.markings,decorations.pathreplacing,decorations.pathmorphing,shadows,positioning,arrows.meta,matrix,fit}

\newcommand{\E}{\mathbb{E}}
\usepackage{mkolar_definitions}

\definecolor{lightgray}{rgb}{0.83, 0.83, 0.83}
\usepackage{colortbl}

\newcommand{\kl}{D_{\text{KL}}}
\newcommand*{\QEDA}{\null\nobreak\hfill\ensuremath{\blacksquare}}%
\title[Regularized DeepIV with Model Selection]{Regularized DeepIV with Model Selection} 
\usepackage{times}
\coltauthor{%
 \Name{Zihao Li}\thanks{Equal contribution} \Email{zl9045@princeton.edu}\\
 \addr Princeton University
 \AND
 \Name{Hui Lan}\thanks{Equal contribution} \Email{huilan@stanford.edu}\\
 \addr Stanford University
 \AND
 \Name{Vasilis Syrgkanis}
 \Email{Vasilis Syrgkanis@stanford.edu}\\
 \addr Stanford University
 \AND
 \Name{Mengdi Wang} \Email{mengdiw@princeton.edu}\\
 \addr Princeton University
 \AND
 \Name{Masatoshi Uehara}\thanks{This work was done at Cornell University} \Email{ueharam1@gene.com}\\
 \addr Genentech 
}

\begin{document}

\maketitle

\begin{abstract}%
In this paper, we study nonparametric estimation of instrumental variable (IV) regressions. While recent advancements in machine learning have introduced flexible methods for IV estimation, they often encounter one or more of the following limitations: (1) restricting the IV regression to be uniquely identified; (2) requiring minimax computation oracle, which is highly unstable in practice; (3) absence of model selection procedure.  In this paper, we present the first method and analysis that can avoid all three limitations, while still enabling general function approximation. Specifically, we propose a minimax-oracle-free method called Regularized DeepIV (RDIV) regression that can converge to the least-norm IV solution. Our method consists of two stages: first, we learn the conditional distribution of covariates, and by utilizing the learned distribution,  we learn the estimator by minimizing a Tikhonov-regularized loss function. We further show that our method allows model selection procedures that can achieve the oracle rates in the misspecified regime. When extended to an iterative estimator, our method matches the current state-of-the-art convergence rate. Our method is a Tikhonov regularized variant of the popular DeepIV method with a non-parametric MLE first-stage estimator, and our results provide the first rigorous guarantees for this empirically used method, showcasing the importance of regularization which was absent from the original work.
 %%%%Empirical validation through numerical simulations supports the efficacy of our proposed method.
\end{abstract}


\vspace{-2mm}
\section{Introduction}\label{sec:intro}
\vspace{-2mm}

Instrumental variable (IV) estimation is an important problem in various fields, such as causal inference \citep{angrist1995identification, newey2003instrumental,deaner2018proxy,cui2020semiparametric,kallus2021causal,kallus2022causal}, missing data problems \citep{miao2018confounding,wang2014instrumental}, dynamic discrete choice models \cite{kalouptsidi2021linear}  and reinforcement learning \citep{liao2021instrumental,uehara2022future,uehara2022provably,shi2022minimax,wang2021provably,yu2022strategic}. 

In this paper, we focus on nonparametric IV (NPIV) regression \citep{newey2003instrumental}.  NPIV concerns three random variables $X \in \RR^d$ (covariate), $Y \in \RR$ (outcome variable), and $Z \in \RR^d$ (instrumental variables). 
%%%%%, which all take values in three bounded subsets $\Xcal$, $\Ycal$, and $\Zcal$ of three Euclidean spaces. 
We are interested in finding a solution $h_0$ of the following conditional moment equation \citep{dikkala2020minimax, chernozhukov2019inference}:   $$
\E[Y - h(X)|Z] = 0.
$$
This is equivalently written as $\Tcal f=r_0$ where $\Tcal: L_2(X)\ni f(X) \mapsto \E[f(X)|Z]\in L_2(Z)$ and $r_0(Z) = \E[Y|Z]$ by denoting $L_2(X), L_2(Z)$ to be the $L_2$ space defined on $X$ and $Z$ with respect to the underlying distribution. Both the operator $\Tcal$ and $\E[Y|Z]$ remain unknown. Hence, we aim to solve $\Tcal f=r_0$ by harnessing an identically independent distributed (i.i.d.) dataset $\{X_i, Y_i, Z_i\}_{i\in[n]}$. 

\iffalse 
We define $L_2(X), L_2(Z)$ to be the $L_2$ space defined on $X$ and $Z$ with respect to the underlying probability distribution.
We adopt the notation of $\Tcal: L_2(X) \rightarrow L_2(Z)$, defined by $\Tcal f(Z) = \E[f(X)|Z]$, and $r_0(Z) = \E[Y|Z]$, the aforementioned equation can then be expressed as $\Tcal h = r_0$. Note that here both the operator $\Tcal$ and $\E[Y|Z]$ remain unknown, and we only have access to an identically independent distributed dataset $\{X_i, Y_i, Z_i\}_{i\in[n]}$.
\fi

There has been a surge in interest in NPIV regressions that try to integrate general function approximation such as deep neural networks beyond classical nonparametric models \citep{hartford2017deep,singh2019kernel,xu2021deep,zhang2023instrumental,dikkala2020minimax,bennett2020variational,bennett2023minimax,bennett2023source,kallus2022causal,singh2020kernel}. 
%%%%Several classical works have proposed sieve or kernel-based estimators for NPIV \citep{singh2020kernel,newey2003instrumental,newey2013nonparametric,horowitz2011applied,carrasco2007linear}.
Despite these extensive efforts, existing approaches encounter several challenges. The first challenge is the ill-posedness of the inverse problem. Many existing works \citep{liao2020provably,newey2003instrumental,florens2011identification,kato2021learning} require that the NPIV solution $h_0$ is unique, and further impose quantitative bounds on measures of ill-posedness.
However, it is known that the uniqueness assumption is easily violated in practical scenarios, such as weak IV \citep{andrews2005inference, andrews2019weak} or proximal causal inference \citep{kallus2021causal}.
The second challenge involves the reliance on minimax optimization oracles in many methods \citep{bennett2023minimax,dikkala2020minimax,liao2020provably,bennett2023source,zhang2023instrumental}, which results in minimax non-convex non-concave optimization when invoking deep neural networks. However, currently, such an optimization can be notoriously unstable and may fail to converge \citep{lin2020near,jin2020local,lin2020gradient,diakonikolas2021efficient,razaviyayn2020nonconvex}. Instead, our approach seeks to address this challenge by proposing a computationally efficient estimator that relies on standard supervised learning oracles rather than minimax oracles. The third challenge is the absence of clear procedures for model selection in existing works \citep{xu2021deep,zhang2023instrumental,cui2020semiparametric,hartford2017deep}. This issue is problematic, because model selection, including techniques like cross-validation, has played a pivotal role in the practical success of machine learning algorithms \citep{bartlett2002model,gold2003model,guyon2010model,cawley2010over,raschka2018model,emmert2019evaluation,mcallester2003pac}. Model selection becomes essential particularly in scenarios where the true NPIV solution $h_0$ lies outside the chosen function classes optimized by the algorithm, which has been seldom explored in prior works. 

\begin{table*}[!t]
    \centering
      \caption{Summary of IV regression literature with general function approximation such as neural networks. ``Model Selection'' means allowing model selection methods. ``No Minimax'' means no need of minimax oracle.  ``No Uniquness'' means unique solution is not assumed.}
    \begin{tabular}{cccccc}
    \toprule
         & Model Selection & No Minimax  & No Uniqueness & RMSE rates\\ \midrule
   \citet{hartford2017deep}      &   & \checkmark  &  & \\  
   \citet{DikkalaNishanth2020MEoC} &  & & & \\ 
   \citet{liao2020provably} &  & & & \checkmark  \\ 
   \citet{xu2021deep} &  & \checkmark  & \\ 
\citet{bennett2023minimax} &  & & \checkmark & \checkmark \\ 
   \citet{bennett2023source} &  & & \checkmark & \checkmark \\ 
\rowcolor{lightgray}
 Ours & \checkmark  &  \checkmark  & \checkmark  & \checkmark \\
   \bottomrule
    \end{tabular}

    \label{tab:all_summary}
\end{table*}

To address aforementioned challenges, we propose a two-stage method, which we refer to as the \emph{Regularized DeepIV (RDIV)}. This approach consists of two steps. First, we learn the operator $\Tcal$ by maximum likelihood estimation (MLE). Secondly, we obtain an estimator for $h_0$ by solving a loss incorporating the learned $\Tcal$ and Tikhonov regularization \citep{ito2014inverse} to handle scenarios where solutions of the conditional moment constraint are nonunique. While our method can be viewed as a regularized variant of the DeepIV method of \cite{hartford2017deep} with a non-parametric MLE first-stage, no prior theoretical convergence guarantees exist for the DeepIV method. We show that our estimators can converge to the least norm IV solution (even if solutions are nonunique) and derive its $L_2$ error rate guarantee based on critical radius. Subsequently, we introduce model selection procedures for our estimators. Particularly, we provide theoretical guarantees for model selection via out-of-sample validation approaches, and show an oracle result in our context. Finally, we further illustrate that our method can be easily generalized to an iterative estimator that more effectively leverages the well-posedness of $h_0$. 

% \masa{I think we shouldn't talk more about repeteive estimator cuz this part doesn't have any model selection guarantee. }
% \masa{Probally, more add meat in this paragraph}

Our contribution is to propose the first estimator for NPIV that (a) operates in the absence of the uniqueness assumption, (b) does not rely on the minimax computational oracle, and (c) allows for model selection. Subsequently, we demonstrate that our estimator can be extended to an iterative estimator, which achieves a state-of-the-art convergence rate in terms of $L_2$ error analogous to  \citet{bennett2023source}, while \citet{bennett2023source} requires a minimax computational oracle and does not permit us to perform model selection. Therefore, our estimator can be seen as an estimator with a strong theoretical guarantee due to the property (a) while it is practical due to properties (b) and (c). Notably, none of the existing works can enjoy such a guarantee, as shown in Table~\ref{tab:all_summary}. 

%%%%%Our contributions are threefold: (1) We target the least norm solution $h_0$. This is a standard approach for inverse problems with non-unique solutions \cite{florens2011identification,bennett2023source,bennett2023minimax}, and allows us to dismiss the uniqueness assumption. (2) By learning the operator $\Tcal$ with density estimation methods, our estimator does not require a minimax computation oracle. (3) By XXX, we illustrate that our method allows the model selection procedure.... To justify our method from an empirical standpoint, we further conduct numerical simulation, with neural network as the general function class. To the best of our knowledge, our method is the first one that is uniqueness assumption and minimax-oracle free, and allows model selection approaches.
%%%MU: I delete the above 

% (Fourth problem: explain our papers in detail. ) 
% To solve the above problems, we propose..... 

% (Fifth paragraph: Summarize our contribution brifely, following the table. We propose the first approach that does not require a minimax oracle without the uniqueness assumption. Then, we propose the first model-selection approach.)  

\vspace{-3mm}
\section{Related Works}
\vspace{-3mm}

\paragraph{Nonparametric IV problem.} 
% \masa{Can be much more concise by focusing on comparing to our work. Generally, we should't move the related work to Appendix because it is important to show our novelty. }
Nonparametric IV estimation has been extensively explored in past decades. Such estimation is tough to solve even when both the linear operator $\Tcal$ and the response $r_0$ are known, known as ill-posedness. The ill-posedness often refers to the presence of one or more
of the following characteristics: (1) the absence of solutions, (2) the existence of multiple solutions, and (3) the discontinuity of the inverse of operator $\Tcal$. Many traditional nonparametric estimators have been proposed to address these challenges, such as series-based estimators \citep{florens2011identification,ai2003efficient,chen2021robust,chen2012estimation,darolles2011nonparametric} and kernel-based estimators \citep{hall2005nonparametric,horowitz2007asymptotic,singh2019kernel}. However, these methods cannot directly accommodate modern machine-learning techniques like neural networks.

Recently, there has been growing interest in the application of general function approximation techniques, such
as deep neural networks and random forests, to IV problems in a unified manner. Among those methods, \citet{bennett2020variational,dikkala2020minimax,lewis2018adversarial, liao2020provably,zhang2023instrumental} reformulate the conditional moment constraint into a minimax optimization and use its solution as the estimator. Notably, \citet{liao2020provably,bennett2023source,bennett2023minimax} establish $L_2$ convergence by linking minimax optimization with Tikhonov regularization under the assumption of the source condition. Moreover, \citet{liao2020batch} assumes uniqueness of solution $h_0$. \cite{dikkala2020minimax,lewis2018adversarial} provide a guarantee for the projected MSE without further assumptions. However, they could not guarantee the convergence rate in strong $L_2$ metric when multiple solutions to conditional moment constraint exist. Furthermore, these methods require a computation oracle for minimax optimization, which further makes model selection challenging. 
%%%further hinders performing model selection when multiple function classes are involved. %%%%Indeed, it is still unknown how to perform model selection given this minimax regression procedure. 
In contrast, our method does not require computational oracles and enables model selection with statistical guarantees. 

Several existing works eschew the need for minimax optimization oracles \citep{hartford2017deep, xu2021deep}. However, all these works do not provide finite sample guarantee or model selection. For example, as the most related work, DeepIV \citep{hartford2017deep} introduces a similar loss function to us. However, it 
lacks an explicit regularization term, which results in the lack of theoretical guarantee and the lack of guarantee for model selection. As another work, \cite{xu2021deep} extends the two-stage kernel algorithm in \cite{singh2019kernel} to deep neural networks, but their algorithm is essentially a bilevel optimization problem, which is hard to solve in general \citep{hong2023two,khanduri2021near,guo2021randomized}. 

\paragraph{Model selection.} 

Model selection has been well studied in the regression and supervised machine
learning literature \citep{bartlett2002model,gold2003model,mcallester2003pac}. The objective can be described more concretely as follows: given $M$ candidate models, $\{f_1,\dots,f_M\}$, each having some statistical complexity $\delta_j$ and some approximation error $\epsilon_j$ (with respect to some un-known true model $f_0$) we wish to find an aggregated model $\hat{f}$ whose mean squared error is closed to the optimal trade-off between statistical complexity and approximation error among all models, i.e.:
$\|\hat{f} - f_0\| \lesssim \min_{j=1}^M \delta_j + \epsilon_j.$
The statistical complexity of a function space can be accurately characterized, albeit the approximation error is un-attainable as it relates to the unknown true model. A guarantee of the form above implies that using the observed data we can compete (up to constants) with an oracle that knows the approximation errors and chooses the best model space. We leave the detailed summary of existing works in Appendix \ref{sec:related-works}. Despite the abundance of methodologies for IV regression problems, few studies have investigated model misspecification and provided model selection procedures to select the best model class. As a few exceptional works, while \cite{xu2021deep} and \cite{AI20075} considered the misspecified regime, but they did not discuss model selection approaches. A typical approach to model selection is out-of-sample validation: estimate different models on half the data and select the estimated model that achieves the smallest empirical risk on the second half (or the best convex ensemble of models that achieves the smallest out-of-sample risk). One problem that arises for model selection in this IV regression setup is to transform the excess risk guarantees, which will be in terms of the weak metric, i.e. $\|\Tcal (\cdot)\|_2$, into the desired bound in the $L_2$ error. In this work, we show that by leveraging the Tikhonov regularization, we can achieve an MSE bound that achieves the same order as the oracle function class. 
\vspace{-3mm}
\section{Notations}\label{sec:notations}
\vspace{-3mm}

For a function $f: \Xcal \times \Ycal \times \Zcal \rightarrow \RR$, 
we denote its population expectation by $\E[f(X,Y,Z)]$. We denote the empirical mean of $f$ by $\E_n[f(X,Y,Z)]:= \frac{1}{n}\sum_{i=1}^n f(X_i, Y_i, Z_i)$. We denote the set of all probability distributions defined on set $\Omega$ by $\Delta(\Omega)$. We denote the $L_p$ norm of $f$ by $\|f\|_p := \E[|f|^p]^{1/p}$. Throughout the paper, whenever we use a generic norm of a function $\|f\|$, we will be referring to the $L_2$-norm. For two density function $p(x)$ and $q(x)$, we denote their Hellinger distance by $H(p(\cdot)\mid q(\cdot)) = \int_\Xcal (\sqrt{p(x)} - \sqrt{q(x)})^2 d\mu(x)$. For a functional operator $\Tcal: L_2(X) \rightarrow L_2(Z)$, we denote the range space of $\Tcal$ by $\Rcal(\Tcal )$, i.e., $\Rcal(\Tcal) = \{\Tcal h: h \in L_2(X)\}.$ Moreover, we use $\Tcal^*: L_2(Z) \rightarrow L_2(X)$ to denote the adjoint operator of $\Tcal$, i.e.,$\langle g, \Tcal h\rangle_{L_2(Z)} = \langle \Tcal^* g, h\rangle_{L_2(X)}$ 
for any $h \in L_2(X), g \in L_2(Z)$, where $\langle\cdot,\cdot\rangle_{L_2(X)}$ and
$\langle\cdot,\cdot\rangle_{L_2(Z)}$ are inner products over $L_2(X)$ and $L_2(Z)$, respectively. For $\theta\in\Theta = \{\theta | \sum_j \theta_j = 1, \theta_j\geq 0,\forall j\} $, we denote $h_{\theta} = \sum_j \theta_j h_j$. We use $e_j$ to denote the one-hot vector where that is zero except for the $j^{th}$ component, which equals to $1$. For a function class $\Fcal$, we define the localized Rademacher complexity by $
\bar{R}_n (\delta; \Fcal) := \E\big[ \E_\epsilon\big[\sup_{f\in \Fcal, \|f\|_2 \leq \delta} \big|\frac{1}{n}\sum_{i=1}^n \epsilon_i f(x_i, z_i)\big|\big]\big],$
where $\epsilon_i$ are i.i.d. Rademacher random variables. For a function class $\Fcal$ over $\Xcal$ and $\Zcal$, we define its star hull by $\operatorname{star}(\Fcal) = \{\gamma f, \gamma \in [0,1], f\in \Fcal \}$. For a function class $\Fcal$, we denote $\bar{\Fcal} := \operatorname{star}(\Fcal - \Fcal)$ to define its symmetrized star hull. We define the critical radius $\delta_{n,\Fcal}$ of a function class $\Fcal$ as any solution to the inequality $\delta^2 \geq \bar{R}_n(\operatorname{star}(\Fcal - \Fcal),\delta)$.  We use $\mu$ to denote the Lebesgue measure. 

% In our work, note all of the proofs are deferred to the Appendix unless otherwise noted.  
\vspace{-2mm}
\section{Problem Statement and Preliminaries}
\vspace{-2mm}

As mentioned in \pref{sec:intro}, we aim to solve the following inverse problem with respect to $h$, known as the nonparametric IV regression:
\begin{align}\label{eq:cond-mom-condi} \textstyle 
   \Tcal h = r_0 ,\quad r_0:= \E[Y|Z].  
\end{align}
While $\Tcal$ and $r_0$ are unknown a priori, using i.i.d.  observations $\{X_i, Y_i, Z_i\}_{i\in[n]}$, we aim to solve this equation. We denote its associated distributions by $g_0$, e.g., denote the conditional density of $X\in\Xcal$ given $Z\in\Zcal$ by $g_0(x|z) \in \{\Xcal \times \Zcal \rightarrow \RR\}$. 
Throughout this work, we assume a solution to Equation \eqref{eq:cond-mom-condi} exists.
\begin{assumption}[Existence of Solutions]\label{ass: sol-exist}
    We have $r_0 \in \Rcal(\Tcal)$, i.e.  $\Ncal_{r_0}(\Tcal):=\{h\in \Hcal: \Tcal h = r_0\} \neq \varnothing.$
\end{assumption}
Crucially, even though a solution to \eqref{eq:cond-mom-condi} exists, it might not be unique. Hence, we propose to target a specific solution that achieves the least norm, defined as:
\begin{align}\label{eq:def-ls}\textstyle 
    h_0 := \argmin_{h\in \Ncal_{r_0}(\Tcal)} \|h\|_2.
\end{align}
Note this least norm solution is well-defined, as it is defined by the projection of the origin onto a closed affine space
$\Ncal_{r_0}(\Tcal) \subset L_2(X)$. Indeed, with Assumption \ref{ass: sol-exist}, it is easy to prove that $h_0$ in \eqref{eq:def-ls} always exists \citep[Lemma 1]{bennett2023minimax}. 

As we emphasize the challenges in \pref{sec:intro}, although there have been a lot of method  that use minimax optimization for estimating $h_0$, when using general function approximation such as neural networks, the minimax optimization tends to be computationally hard \citep{lin2020near,jin2020local,lin2020gradient,diakonikolas2021efficient,razaviyayn2020nonconvex}.  Moreover, it remains unclear how to perform model selection for those methods. Hence, in this paper, we aim to propose a new method that can incorporate any function approximation for estimating the least square norm solution $h_0$ in \eqref{eq:def-ls} with a strong convergence guarantee in $L_2(X)$ under mild assumptions (i.e., such as without the uniqueness of $h_0$) while allowing for model selection.

\vspace{-2mm}
\section{Regularized Deep IV}\label{sec:method-intro}
\vspace{-2mm}

In this section, we introduce a two-stage algorithm, Regularized DeepIV (RDIV), aimed at obtaining the least square solution $h_0$ as defined in Equation \eqref{eq:def-ls}. Even though we borrow the DeepIV terminology from the prior work \citep{hartford2017deep}, our method can be used with arbitrary function approximators and not necessarily neural network function spaces.
Being inspired by the original constrained optimization \eqref{eq:def-ls}, we aim to solve a regularized version of the problem: \begin{align}\label{eq:pop-reg-noniter}
h_{*} := \argmin_{h\in \Hcal} \|Y-\Tcal h \|_2^2 + \alpha \|h\|_2^2
\end{align}
where $\Hcal \subset L_2(X)$ represents a hypothesis class that consists of possible candidates for $h_0$, and $\alpha\in \RR^+$ denotes a parameter controlling the strength of regularization. 
%%%It's well-known that for a suitably chosen $\alpha$, Equation \eqref{eq:pop-reg-noniter} is equivalent to the original constrained optimization problem in Equation \eqref{eq:def-ls} \citep{cavalier2011inverse,mendelson2010regularization}.
% Remarkably, we do not assume realizability, i.e. $h_0 \in \Hcal$,  which often fails under many cases, e.g. $\Hcal$ is a neural network.  \zihao{add some reference}
% Specifically, we can set $h_{-1,*} = 0$. 
While this formulation itself has been known in the literature on general inverse problems \citep{cavalier2011inverse,mendelson2010regularization}, we consider common scenarios in IV   where both the conditional expectation operator $\Tcal$ and the population expectation in Equation \eqref{eq:pop-reg-noniter} are unknown, and need to leverage dataset $\{X_i,Y_i,Z_i\}$.  

\begin{algorithm}[!t] 
\caption{Regularized Deep IV (RDIV) }

\begin{algorithmic}[1]\label{alg:dbiv-noniterative}
        \REQUIRE Validation dataset $\{X_i,Y_i,Z_i\}_{i\in[n']}$ that is independent from the training dataset, function class $\Gcal\subset \{\Zcal \to \Delta(\Xcal)\}$, function class $\Hcal \subset \{\Xcal \to \RR\}$, a regularization hyperparameter $\alpha \in \RR_{>0}$
        \STATE Learn $\hat{g}(x|z)$ with MLE: 
        \begin{align}\label{eq:mle-est}
    \hat{g} = \argmax_{g\in\Gcal} \E_n[\log g(X|Z)],
\end{align} 
        \STATE   Learn $\hat{h}$ by the following estimator:  
        \begin{align}\label{eq:emp-reg-noniter}
    \hat{h}= \argmin_{h\in\Hcal} \E_n[\big(Y - (\hat{\Tcal} h)(Z)\big)^2]+ \alpha \cdot \E_n[h(X)^2]
    \end{align}
    where  $\hat{\Tcal}: L_2(X) \rightarrow L_2(Z)$ is defined by $\hat{\Tcal} f(Z) = \E_{x\sim \hat{g}(X|Z)}[f(X)]$ using $\hat g$ in the first step. 
    %     \STATE \begin{align}\label{eq: emp-iter}
    % \hat{h}_{m}=&\argmin_{h\in\Hcal} \E_n[\big(Y - \hat{\Tcal} h(Z)\big)^2]\nonumber\\
    % &\qquad \qquad+ \alpha \cdot \E_n[\big(h(X) - \hat{h}_{m-1}(X)\big)^2],
    % \end{align}
    \OUTPUT $\hat{h}$.
\end{algorithmic}
\end{algorithm}

To address this challenge, by integrating general function approximation such as neural networks, we propose a two-stage method, the \emph{Regularized Deep Instrumental Variable (RDIV)}, which is summarized in \pref{alg:dbiv-noniterative}. In the first stage, given a function class $\Gcal$ comprising functions of the form $\big\{g: \Xcal\times\Zcal \rightarrow \RR,\int_\Xcal g(x|z) \mu(dx) =1 \text{ for all } z\big\}$, we aim to learn the conditional expectation operator $\Tcal$ by estimating the ground-truth conditional density $g_0(x|z)$ from the dataset $\{X_i, Z_i\}_{i\in[n]}$ with MLE in Equation \eqref{eq:mle-est}. In the second stage, with the learned conditional density $\hat{g}$ in the first step, we learn $h_0$ by replacing expectation and $\Tcal$ in Equation \eqref{eq:pop-reg-noniter} with empirical approximation and $\hat \Tcal$, respectively, as shown in Equation \eqref{eq:emp-reg-noniter}.

Importantly, our method does not necessitate a demanding computational oracle such as non-convex non-concave minimax or bilevel optimization, unlike many existing works for nonparametric IV with general function approximation  \citep{lewis2018adversarial,xu2021deep, bennett2023minimax}. Even when using neural networks for $\Gcal$ and $\Hcal$, we just need standard ERM oracles for density estimation or regression whose optimization is empirically known to be successful and theoretically more supported \citep{du2019gradient,chen2018convergence,zaheer2018adaptive,barakat2021convergence,wu2019global,zhou2018convergence,ward2020adagrad}.  We leave the numerical comparison between our method and existing NPIV methods \citep{hartford2017deep, dikkala2020minimax,xu2021deep,singh2019kernel} in Appendix \ref{sec:experiment}.

\begin{remark}[Comparison with Deep IV]
Our algorithm shares similarities with DeepIV in \citep{hartford2017deep}, and indeed, it draws inspiration from it. However, a key distinction lies in our introduction of an explicit regularization term in Equation~\eqref{eq:emp-reg-noniter}. Such a term endows the loss function with strong convexity, which plays a pivotal role in obtaining guarantees without the requirement for solution uniqueness. Furthermore, the original DeepIV work lacks a rigorous discussion on convergence guarantees or model selection. Hence, despite the algorithmic resemblances, our contributions primarily focus on the theoretical aspect, showcasing rapid convergence guarantees under mild assumptions, linking them to a formal model selection procedure, and exploring the iterative version to achieve a refined rate in \pref{sec: iterative}.
\end{remark}

\begin{remark}[Computaion for $\hat \Tcal$]
Some astute readers might notice it could be hard to evaluate $\hat{\Tcal}h $ exactly in Equation \eqref{eq:emp-reg-noniter}. However, in practical application when $h$ is parametrized as a neural network, we can sample a batch of $\{X_j'\}_{j\in[B]}$ by $\hat{g}(X|Z_i)$ for every $Z_i$ in the dataset, and calculate a stochastic gradient that is an unbiased estimator of the real gradient of the loss function in Equation \eqref{eq:emp-reg-noniter}. Existing theory and empirical results for stochastic first-order methods can then guarantee the performance in many scenarios \citep{,jin2019nonconvex,barakat2021convergence,chen2018convergence,hartford2017deep}.

% However, we can easily approximate $\hat{\Tcal}$ with Monte Carlo method by generating i.i.d. samples from $\hat{g}(x|z)$. Especially when we estimate $\hat{\Tcal}$ with $O(n^2)$ samples, the errors due to Monte Carlo approximation are of order $O(\delta_{n,\Hcal}) \sim O(\|\hat \Tcal f - \Tcal f\|)$,  which are negligible compared to primal error terms. 
\end{remark}

\vspace{-2mm}
\section{Finite Sample Guarantees}\label{sec: finite-sample}
\vspace{-2mm}

In this section, we demonstrate a convergence result of our estimator $\hat{h}$ in RDIV to $h_0$ and derive its $L_2$ error rate after introducing several assumptions. 
%%%The proofs of Theorems \ref{thm: mle-nonmisspec-noniter} and \ref{thm:chi2-mle-result-noniter} are deferred to Appendix \ref{sec:proof-mle-nonmisspec-noniter}, while the proofs of Theorems \ref{thm: mle-misspec} and \ref{thm: chi2mle-misspec} are provided in Appendix \ref{sec: proof-misspec}.

We commence by introducing the $\beta$-source condition, a concept commonly used in the literature on inverse problems \citep{carrasco2007linear,ito2014inverse,engl1996regularization,bennett2023source,liao2021instrumental}, which mathematically captures the well-posedness of the function $h_0$.
\begin{assumption}[$\beta$-Source Conditon]\label{ass:source-cond}
The least norm solution $h_0$ satisfies $ h_0 = (\Tcal^*\Tcal)^{\beta/2}w_0$ for some $w_0\in {\Hcal}$ and $\beta \in \RR_{\geq 0}$, i.e.,$h_0 \in \Rcal(\Tcal^*\Tcal )^{\beta/2}$. Recall $\Tcal^*$ is an adjoint operator of $\Tcal$ defined in \pref{sec:notations}. 
\end{assumption}

In the following, we present its interpretation. First, as special cases, when $\Xcal,\Zcal$ are finite (e.g., discrete random variables), it holds when $\beta=\infty$. However, in our cases of interests where $\Xcal,\Zcal$ are not finite, this assumption restricts the smoothness of $h_0$. 
Intuitively, when the parameter $\beta$ is large, the function $h_0$ exhibits greater smoothness, and the assumption gets stronger, in the sense that eigenfunctions of $h_0$ relative to an operator $\Tcal$ have smaller eigenvalues as explained in \citet[Section 6.4]{bennett2023minimax}.   
%%%%Source conditions are commonly used to derive strong convergence rate guarantees in the inverse problem literature. 
%%%Such conditions have been widely used for both inverse problems with known operators \citep{ito2014inverse,engl1996regularization} and IV problems with unknown operators \citep{bennett2023minimax, bennett2023source,liao2021instrumental,florens2011identification}.

Next, we introduce another standard assumption as follows. This requires that the function classes $\Hcal$ and $\Gcal$ are well-specified. We will later consider misspecified cases as in \pref{sec:misspecified}. 
\vspace{-2mm}
\begin{assumption}[Realizability of function classes]\label{ass: realizability}
    We assume $h_0\in\Hcal$, $g_0\in \Gcal$.
\end{assumption}
\vspace{-2mm}

% We organize the rest of this section in two parts. In Section \ref{sec:mle-noniterative}, we discuss the convergence rate for MLE-based estimator, and in Section \ref{sec:chi2mle-noniterative}, we discuss the convergence rate for $\chi^2$-MLE estimator.
%%%%\subsection{Results with Vanilla MLE}\label{sec:mle-noniterative} \masa{The original subsection sounds wired, I changed }
%%%%We first discuss a finite sample convergence rate of  Algorithm \ref{alg:dbiv-noniterative} when using  MLE as $\hat g$ (i.e., using \eqref{eq:mle-est}). 

%%%%Such a theorem assumes that the conditional density is uniformly bounded away from zero on $\Xcal$. 
% Recent studies such as \citet{bilodeau2023minimax} have shown how to analyze MLE without such an assumption, however, to convey clarity in our narrative, we choose to retain this assumption.\masa{I don't think we need this last sentence. }

The final assumption is as follows. This is standard in analyzing the convergence of nonparametric MLE \citep[Chap 14, p.g. 476]{wainwright2019high}. We will later discuss how to relax such an assumption in Remark~\ref{rem:bounded} and Appendix~\ref{sec:chi2}. 
\iffalse 
\footnote{Several works \citep{bilodeau2023minimax,Zhang_2006} have shown how to analyze MLE without such an assumption. However, to convey our main contributions clearly without further complications, we choose to retain them. We will later discuss how to relax such an assumption in Remark~\ref{rem:bounded} and Appendix~\ref{sec:chi2}. }.
\fi 

\vspace{-2mm}
\begin{assumption}[Lower-bounded density]\label{ass:lower-bound}
    We assume a constant $C_0>0$ such that $g_0(x|z) > C_0$ holds for all $x \in \Xcal$ and $z\in\Zcal$. 
\end{assumption}
\vspace{-2mm}
% The following moment comparison assumption bounds $L_2$ norm by $L_4$ norm from below.
% \begin{assumption}[(Moment Comparison]\label{ass:bound-2-4}
%      There is a constant $\cmc$ such that for all $h - h' \in \Hcal - \Hcal$, we have $$
%      \| h - h'\|_4 \leq \cmc \| h - h'\|_2.
%      $$
% \end{assumption}
% The moment comparison condition has been used in statistics as a minimal assumption for learning without boundedness \cite{lecue2013learning,mendelson2015learning,liang2015learning}.

Finally, we present our guarantee for Algorithm \ref{alg:dbiv-noniterative}. 

\vspace{-2mm}
\begin{theorem}[$L_2$ convergence rate for RDIV with  MLE]\label{thm: mle-nonmisspec-noniter}
    Suppose Assumption \ref{ass:source-cond},\ref{ass: realizability},\ref{ass:lower-bound} hold. Let $\|Y\|_\infty\leq C_Y$, $\|h\|_\infty \leq C_\Hcal$ holds for all $h\in\Hcal$, $\|g\|_\infty \leq C_\Gcal$ holds for all $g\in\Gcal$. 
    There exists absolute constant $c_1, c_2$, such that with probability at least $1- c_1\exp(c_2n\delta_n^2)$: 
    \begin{align*}\textstyle 
           \| \hat{h} - h_0 \|_2^2 =   O(\underbrace{\delta^2_n/\alpha^2}_{\text{(i)}}  +  \underbrace{\alpha^{\min(\beta,2) }}_{\text{(ii)}}   ) 
    \end{align*}
  In particular, by setting $\alpha = \delta_n^{\frac{2}{2+\min\{\beta, 2\}}}$
    we have 
    \begin{align}\label{eq:final}\textstyle 
    \| \hat{h} - h_0 \|_2^2 =  O\big(\delta_n^{\frac{2\min\{\beta, 2\}}{2 + \min\{\beta, 2\}}}\big).       
    \end{align}
    Here $\delta_{n} = \max\{\delta_{n, \Gcal},\delta_{n, \Hcal} \}$, where $\delta_{n, \Fcal}$ is the critical radius of $star(\Fcal - \Fcal) = \{\lambda (f - f'), f,f'\in\Fcal, \lambda \in [0,1]\}$. $O(\cdot)$ hides constants of polynomial order of $C_Y, C_\Gcal, C_\Hcal, $ and $1/C_0$.
\end{theorem}

\vspace{-2mm}
\paragraph{Sketch of the proof.}
We now sketch the proof of Theorem \ref{thm: mle-nonmisspec-noniter}. Recall that $h_*$ is the optimizer of the Tikhonov-regularized loss function \eqref{eq:pop-reg-noniter}. We first introduce the following lemma, which characterizes the bias caused by the regularization. 
\begin{lemma}[Regularization Bias]\label{lem:bias-noniter}
    Under Assumption \ref{ass:source-cond}, we have $$
    \|h_* - h_0\|_2^2 =O\big(\alpha^{\min\{\beta,2\}}\|w_0\|_2^2\big).
    $$
\end{lemma}

Therefore, recalling $\|\hat h- h_0 \|_2^2 \leq 2(\|h_* - h_0\|_2^2 + \|\hat{h} - h_*\|_2^2)$, we only need to bound $\|\hat{h} - h_*\|_2^2$. Utilizing the strong convexity of \eqref{eq:pop-reg-noniter}, we have the following lemma:
\begin{lemma}[Empirical Deviation \& First-stage Bias]\label{lem: bound-term2}
    With probability at least $1- c_1 \exp(c_2n\delta_{n,\Hcal}^2)$, we have the following inequality: 
    \begin{align*}
        \|\hat h - h_* \|^2_2 
        \leq \frac{1}{\alpha}\bigg\{\underbrace{|(\EE_n-\EE)  [L(\hat \Tcal \hat h) - L(\hat{\Tcal} h_*) ]|}_{\text{(a1)}}+ \underbrace{\|(\hat \Tcal - \Tcal)(\hat h - h_*)\|_1}_{\text{(a2)}}\bigg\},
    \end{align*}
    where $L(f) := (Y - f(Z))^2$.
\end{lemma}
Lemma \ref{lem: bound-term2} shows we can bound $\|\hat{h} - h_{*}\|_2^2$ by two terms (a1) and (a2). Here (a1) is a centered empirical process, and  (a2) is the error when estimating $\Tcal$ by $\hat\Tcal$. Utilizing localized concentration inequality and the boundedness of function class $\Hcal$, we can bound (a1) by $O(\delta_{n,\Hcal}^2 + \delta_{n,\Hcal}\|\hat{h} - h_*\|_2)$. To control (a2), we prove the following lemma: \begin{lemma}[MLE error]\label{lem:operator-error}
With probability at least $1-\exp(n\delta^2_{n,\Gcal})$, we have 
$$\|(\hat \Tcal-\Tcal)(h- h')\|_1 \leq \{C_{\Gcal}/C_0+1\}\cdot \delta_{n,\Gcal} \|h- h'\|_2    $$
for every $h- h' \in \Hcal - \Hcal$.
\end{lemma}
Note that such a bound is nontrivial since the standard analysis for MLE only results in a convergence rate in terms of Hellinger distance between the $\hat{g}$ and $g$, and does not directly provide a bound on $L_1$ norm of $\Tcal h - \hat \Tcal h$. With Lemma \ref{lem:operator-error}, we can now bound (a2) from above with the critical radius of $\Gcal$, and the $L_2$ distance between $\hat h$ and $h_*$. Finally, combining current arguments, Lemma \ref{lem: bound-term2} and \ref{lem:operator-error},  we have
\begin{align*}
     \|\hat h - h_* \|^2_2=c'\{\delta_{n,\Hcal}^2 + \delta_{n,\Hcal}\|\hat{h} - h_*\|_2  + \{C_{\Gcal}/C_0+1\}\cdot \delta_{n,\Gcal} \|\hat{h} - h_*\|_2 \} 
\end{align*}
for certain constants $c'$. 
By organizing the above equation, we have \begin{align}\label{eq:bound-term2}
\textstyle  \|\hat{h} - h_0\|^2 = O\bigg(\frac{\max\{\delta_{n,\Gcal}, \delta_{n,\Hcal}\}^2}{\alpha^2}\bigg).
\end{align}
Combine Lemma \ref{lem:bias-noniter} and Equation \eqref{eq:bound-term2}, we further have $$ \textstyle 
\|\hat h - h_0\|_2^2 = O\bigg(\frac{\delta_n^2}{\alpha^2} + \alpha^{\min\{\beta,2\}}\bigg),
$$
select $\alpha = \delta_n^{\frac{2}{2 + \min\{\beta, 2\}}}$ and we conclude the proof. $\QEDA$
    
The critical radius $\delta_n$ measures the statistical complexity of function class $\Hcal$ and $\Gcal$. For example, for parametric class or Gaussian Kernel, $\delta_n =\Tilde{O}(n^{-1/2})$, while for first order Sobolev class, $\delta_n = \tilde{O}(n^{-1/3})$ \citep{wainwright2019high, bartlett2002localized}. In those cases, when $\beta \geq 2$, the final rate in $L_2$ metric will be $\tilde{O}(n^{-1/2})$ in the former case and $\tilde{O}(n^{-1/3})$ in the latter case, respectively. 
We now give the interpretation of our result. The bound of $\|\hat{h} - h_0\|_2^2$ consists of two terms. Term (i) comes from a statistical error to estimate $h_{*}$ from $\Hcal$ and $\Gcal$ (i.e., $\|\hat h- h_{*}\|_2$). Here, we use the strong convexity owing to Tikhonov regularization as it enables us to convert the population risk error to an error in $L_2$ metric as in Lemma~\ref{lem: bound-term2}. Then, we properly bounded the population risk from above  by the empirical process term properly as in Lemma~\ref{lem:operator-error}. While this $\delta^2_n$ rate is known as the standard fast rate in nonparametric regression \citep{wainwright2019high}, our result is still non-trivial because we need to handle a statistical error term properly when approximating $\Tcal$ with $\hat \Tcal$, which comes from the MLE error in the form of Hellinger distance. 

The term (ii) comes from the bias $\|h_0- h_{*}\|_2$ incurred by adding a Tikhonov regularization. This analysis has been used in existing works (e.g., \citep{cavalier2011inverse}). Due to $\min(\beta,2)$, while we cannot leverage a high smoothness $\beta$ especially when $\beta \geq 2$, we will see how to leverage $\beta$ in such a case by introducing an iterative estimator in Section~\ref{sec: iterative}.

%%%Note that (ii) \masa{Is this (i)?} comes from the strong convexity endowed by the Tikhonov regularization, as it enables us to convert the population risk error to an error in $L_2$ metric, the former is further bounded by a concentration inequality for an empirical process. 
%%%%Finally, in Equation \eqref{eq:final}, 
% \masa{if we can elborare more, it is better}
% Notably, the obtained rate in Theorem \ref{thm: mle-nonmisspec-noniter} is equivalent to the current state-of-the-art rates in terms of $L_2$ metric when using general function approximation \citep{bennett2023source,bennett2023minimax}, although the used realizability assumptions are different and we cannot have an exact comparison. 
We also compare our work to existing state-of-the-art convergence rate $O\big(\delta_n^{2\frac{\min\{\beta,1\}}{1+\min\{\beta,1\}}}\big)$ in \cite{bennett2023source}, in which they employ a minimax-type algorithm. When $\beta \geq 2$, we achieve the same rate. We also remark that although our rate is slightly slower than theirs when $\beta \leq 2$, our method does not require a minimax-optimization oracle and can be incorporated with method selection methods. Besides, we will show that our method can achieve a state-of-the-art rate 
in our extension to iterative estimator in \pref{sec: iterative}.

% we have $\|\hat{h} - h_0\|_2^2=O(\delta_n)$, and when $\beta <2$, we have $\|\hat{h} - h_0\|_2^2= O(\delta_n^{\frac{2\beta}{2+\beta}})$. While in \cite{bennett2023source}, when $\beta \geq 1$ they achieve $\|\hat{h} - h_0\|_2^2= O(\delta_n)$, and when $\beta <1$, they have $\|\hat{h} - h_0\|_2^2= O(\delta_n^{\frac{2\beta}{1+\beta}})$. 

% \begin{proof}
%     For a detailed proof, see Appendix \ref{sec:proof-mle-nonmisspec-noniter}. \masa{Just first say we are gonna include all of the proof in Appendix first somewhere. We don't need to repeat this kind of stuff a lot.  }
% \end{proof}
% \zihao{Need remarks here for interpretation?}
% \masa{Say why this result is good. It is the SOTA result compared to XXX. }

\vspace{-2mm}
\begin{remark}[Removing the Boundedness Assumption]\label{rem:bounded}
While the lower-boundedness of density function in  Assumption \ref{ass:lower-bound}  is widely used in existing literature, we can easily remove it in several ways. We show that we can relax it by using a $\chi^2$-\textbf{MLE} instead of MLE: 
 \begin{align}\label{eq:chi2-mle-est}
\hat{g} = \argmin_{g\in\Gcal} \frac{1}{2}\cdot \E_n\bigg[\int_{\Xcal} g^2(x|Z)d\mu(x)\bigg] - \E_n[g(X|Z)]. 
\end{align}
We delay detailed results for our methods under $\chi^2$-MLE in Appendix \ref{sec:chi2}.
\end{remark}
\vspace{-2mm}

\vspace{-2mm}
\section{Misspecified Setting}\label{sec:misspecified}
\vspace{-2mm}

Next, we establish the finite sample result when Assumption \ref{ass: realizability} does not hold, i.e., function classes $\Hcal$ and $\Gcal$ are misspecified. This result serves as an important role in formalizing the model selection procedure in \pref{sec:model-selection}. 
% \masa{Can we get actual rates? }\zihao{For NN, yes we can. However the approximation theory needs additional assumptions for $h_0$ and $g_0$ (Holder smooth etc), could it be too verbose?} \masa{I think we don't need to it formally. But, we can just say in a few sentence without making a lemma or something? }\zihao{Added below the theorem}
% \masa{if we say this stuff, I think we should present actual final rates in those cases when using sieves after these theorems. }

\vspace{-2mm}
\begin{theorem}[$L_2$ convergence rate for RDIV with  MLE under misspecification]\label{thm: mle-misspec}
    Suppose Assumption \ref{ass:source-cond} and \ref{ass:lower-bound} hold, and there exists $h^{\dag}\in \Hcal$ and $g^{\dag} \in \Gcal$ such that $\|h_0 - h^{\dag}\|_2 \leq \epsilon_\Hcal$ and $\E_{z\sim g_0}[\kl(g_0(\cdot|z)\mid g^{\dag}(\cdot|z))] \leq \epsilon_\Gcal$. 
    For any $0<\alpha\leq 1$, we have \begin{align*}
\|\hat{h} - h_0\|_2^2= O\bigg(\underbrace{\frac{\delta_n^2}{\alpha^2}}_{(b1)}+ \underbrace{\alpha^{\min\{\beta+1,2\}-1}}_{ (b2)}+\underbrace{\frac{\epsilon_\Hcal^2}{\alpha} + \frac{  \epsilon_\Gcal}{\alpha^2}}_{(b3)}\bigg)
\end{align*}
holds with probability at least $1- c_1\exp(c_2n\delta_n^2)$. Here $\delta_{n}$ has the same definition in Theorem \ref{thm: mle-nonmisspec-noniter}.
\end{theorem}
% \masa{We should add some explanation, what is (b1), (b2), (b3) something like that. }
% \begin{remark}
% In Theorem \ref{thm: mle-misspec}, the first term  $O(\frac{\delta_n^2 +  \epsilon_\Gcal}{\alpha^2})$ upper bounds $\|\hat{h} - h_*\|$, which comes from the misspecification of $\Gcal$, the statistical error of . The term $O(\alpha^{\min\{\beta,2\}})$ upper bounds the bias caused by regularization , i.e. $\|h_* - h_0\|$. The term  $O(\big(1+ \frac{1}{\alpha}\big) \{\alpha^{\min\{\beta,2\}} + \epsilon_\Hcal^2\})$ comes from the misspecification error of $\Hcal$. Here $(1+\frac{1}{\alpha})$ upper bounds the condition number of the loss function in Equation \eqref{eq:pop-reg-noniter}, which measures how much 
% $h_*$ would shift when the function class $\Hcal$ changes.
% \end{remark}

% \begin{remark}
%     If we assume that there exists a $h^{\ddagger}$ such that $\|h_* - h^{\ddagger}\|_2 \leq \epsilon_\Hcal$, i.e. the distance of the minimizer of \eqref{eq:pop-reg-noniter} , then we get $$
%     \|\hat{h} - h_0\|_2^2 = O\bigg({\frac{\delta_n^2}{\alpha^2}}+ {\alpha^{\min\{\beta,2\}}}+{\frac{\epsilon_\Hcal^2}{\alpha} + \frac{  \epsilon_\Gcal}{\alpha^2}}\bigg),
%     $$
%     which matches the results in Theorem \ref{thm: mle-nonmisspec}. We leave the proof of this remark in Appendix \ref{}.
% \end{remark}
The bound for $\|\hat{h} - h_0\|^2_2$ consists of three terms: term (b1) measures the statistical deviation of a normalized empirical process, term (b2) measures the regularization error caused by Tikhonov regularization and term (b3) measures the effect of model misspecification. Here term (b3) has a poly$(\frac{1}{\alpha})$ dependency. This is because model misspecification causes a higher population risk in both stage 1 and 2 of Algorithm \ref{alg:dbiv-noniterative}. Hence, the more convex the loss function, the lesser the shift in the optimizer. The readers may notice that term (b2) is slightly slower than the original bias term in Theorem \ref{thm: mle-nonmisspec}. This is because the difference of the optimal value in \eqref{eq:pop-reg-noniter} due to misspecification of $\Hcal$ is of order $O(\alpha^{\min\{\beta+1,2\}}+ \epsilon_{\Hcal}^2) $, as we will show in Lemma \ref{lem:misspec-main} in the Appendix. By the $\alpha$-strong convexity endowed by Tikhonov regularization, this results in a shift of $h_*$ of magnitude 
$O\big(\alpha^{\min\{\beta+1,2\}-1}+ \epsilon_{\Hcal}^2/\alpha\big)$.

Theorem \ref{thm: mle-misspec} is particularly useful when we apply estimators based on sample-dependent function classes $\Hcal$ and $\Gcal$ (a.k.a. sieve estimators) that approximate certain function spaces. For example, $\Hcal$ can be linear models with polynomial basis functions that take the form $\langle \phi(X), \theta\rangle$, which can gradually approach H\"{o}lder or Sobolev balls, and $\Gcal$ can be a set of neural networks with a growing dimension \citep{chen2007large,chen2022nonparametric,schmidt2020nonparametric}. 
More specifically, when $X$ and $Z$ are bounded, and $h_0$ and $g_0$ are $s$-H{\"o}lder smooth, it is well known that a deep ReLU neural network with depth $O(\log(1/\epsilon))$, width $O(d\epsilon^{-d/s})$ and weights bounded by $\tilde{O}(1)$ could satisfy the approximation error in Theorem \ref{thm: mle-misspec} \citep{schmidt2019deep}, recall that $d$ is the dimension of $X$ and $Z$ . In that case, $\delta_n^2 = \tilde{O}(\epsilon^{-d/s}/n)$ \citep{bartlett2002localized, chen2022nonparametric}. Choosing the architecture of the neural network according to $\epsilon = \tilde{O}(n^{-1/(1+d/s)})$, then Theorem \ref{thm: mle-misspec} shows that by setting $\alpha = O(n^{\frac{1}{(1+d/\alpha)(\min\{\beta+1,2\}+1)}})$, we have $\|\hat{h} - h_0\|_2^2 = \tilde{O}(n^{\frac{\min\{\beta+1,2\}-1}{(1+d/s)(\min\{\beta+1,2\}+1)}})$.

\iffalse 
\begin{remark}\label{remark: mle-mispec}
    Our result brings insight into how to select $\alpha$ under model misspecification. We define $\epsilon:= \max\{\epsilon_\Gcal, \epsilon_\Hcal^2\}$. For $\epsilon<1$, Then by setting $\alpha = (\delta_n^2 + \epsilon)^{\frac{1}{\min\{\beta,2\}+1}}$, we have $$
    \|\hat{h} - h_0\|_2^2 = O\big((\delta_n^2 + \epsilon)^{\frac{\min\{\beta+1,2\}-1}{\min\{\beta+1,2\}+1}}\big).
    $$ 
    For $\epsilon \geq 1$,  we  set $\alpha = 1$, and $\|\hat{h} - h_0\|^2 = O(\epsilon)$.
\end{remark}
\fi 

\vspace{-2mm}
\section{Model Selection}\label{sec:model-selection}
\vspace{-2mm}

One advantage of employing the proposed two-staged algorithm is that it enables model selection, which is not attainable when a minimax approach is used. In this section, we explain how we perform model selection. We focus on the model selection for the second stage, as the conditional density $\hat{g}$ from the first stage can be selected via existing methods for model selection for maximum likelihood estimators (e.g. \cite{MLEselection,conditiondensitySelection,vijaykumar2021localization}). 

With an MLE-based estimator $\hat{g}$ obtained from the first stage in Algorithm \ref{alg:model_selection}, we consider model selection using the regularized loss in the second stage, with theoretical guarantees in the $\|\cdot\|_{2}$ metric. More concretely, given a choice of $M$ candidate models $\{h_1, \dots, h_M\}$ and a validation dataset $\{X'_i,Y'_i,Z'_i\}_{i=1}^n$ (distinct from the one used for training models $\{h_i\}$ and $\hat{g}$), the goal is for the final output of the model selection algorithm to achieve oracle rates with respect to the minimal misspecification error.

We present our algorithm in Algorithm~\ref{alg:model_selection}. We provide two options for model selection: Best-ERM and Convex-ERM. Best-ERM selects the model that minimizes the regularized loss on a validation set, while Convex-ERM constructs a convex aggregate of the candidate models that minimizes the regularized loss on a validation set.

\begin{algorithm}[!t] \label{alg:model_selection}
\caption{Model Selection for Regularized Deep IV }
\begin{algorithmic}[1]\label{alg:model_selection}
        \REQUIRE Validation dataset $\{X'_i,Y'_i,Z'_i\}_{i\in[n]}$, $M$ candidate models $\{h_i\}_{i=1}^M$, a regularization hyperparameter $\alpha \in \RR_{>0}$, an estimator $\hat g$, which can obtained by MLE with standard model selection procedure in \citet{MLEselection,conditiondensitySelection}. 
        \STATE   Learn $\hat \theta$ with each of the followings:  
        \begin{align}
  \textbf{Best-ERM:}\quad   {\hat \theta} &= \argmin_{\theta= e_1, \dots, e_M} \E_n[\big(Y - (\hat{\Tcal} h_{\theta})(Z)\big)^2]+ \alpha \cdot \E_n[h_{\theta}(X)^2], \\
    \textbf{Convex-ERM:}\quad   \hat{\theta} & = \argmin_{\theta \in \Theta} \E_n[\big(Y - (\hat{\Tcal} h_{\theta})(Z)\big)^2]+ \alpha \cdot \E_n[h_{\theta}(X)^2], 
    \end{align}
      where $h_\theta = \sum_{j=1}^M \theta_i h_i$, $\sum_{j=1}^M \theta_j =1, \theta_j\geq 0$, $\hat{\Tcal} f(Z) = \E_{x\sim \hat{g}(X|Z)}[f(X)]$ and $\EE_n[\cdot]$ is defined for $\{X'_i,Y'_i,Z'_i\}_{i \in [n]}$. 
    %     \STATE \begin{align}\label{eq: emp-iter}
    % \hat{h}_{m}=&\argmin_{h\in\Hcal} \E_n[\big(Y - \hat{\Tcal} h(Z)\big)^2]\nonumber\\
    % &\qquad \qquad+ \alpha \cdot \E_n[\big(h(X) - \hat{h}_{m-1}(X)\big)^2],
    % \end{align}
    \OUTPUT $h_{\hat \theta}$.
\end{algorithmic}
\end{algorithm}

\begin{theorem} [Model Selection Rates] \label{thm:model_selection}
    Consider the model selection problem given $M$ candidate models with any choice of $\alpha$, over $M$ function classes $\{\Hcal_1, \dots, \Hcal_M\}$. Suppose Assumption \ref{ass:source-cond} and \ref{ass:lower-bound} hold, and there exists $g^{\dag} \in \Gcal$ and  $h^{\dag}_j\in \Hcal_j$ for all $j$ such that $\|h_0 - h^{\dag}_j\|_2 \leq \epsilon_{\Hcal_j}$ and  $\E\big[\int_{\Xcal} (g^{\dag}(x|Z) - g_0(x|Z)  )^2 d \mu(x)\big] \leq \epsilon_\Gcal$. Assume that $Y$ is almost surely bounded by $C_Y$, each candidate model $h_j$ is uniformly bounded in $[-C_{\Hcal},C_{\Hcal}]$ almost surely. Let $\delta_{n,j} = \max\{\delta_{n,\Gcal},\delta_{n,\Hcal_j}, \delta_{n,M}\}$, where $\delta_{n,M}$ denotes the critical radius of the convex hull over M variables for Best-ERM (i.e. $\delta_{n,M} = \frac{\log(M)}{n}$), and the critical radius of the set of $M$ candidate functions for Convex-ERM (i.e. $\delta_{n,M} = \frac{M}{n}$). 
    
    With probability $1-c_1 \exp(c_2 n \sum_j^M\delta_{n,j}^2)$, the output of Convex-ERM or Best-ERM $\hat{\theta}$, satisfies:
    % \begin{align*}
    %     \|h_{\hat{\theta}}-h_0\|^2 
    %     % \leq ~&O\big(\min_j (\delta_{n,j} + \epsilon_{\Gcal})^{\frac{\min\{\beta, 2\}}{2 + \min\{\beta, 2\}}} + \epsilon_{\Hcal_j}\big)\\
    %     \leq ~& \min_j O\big( \frac{\delta_{n,j} + \epsilon_{\Gcal}}{\alpha^2} + \alpha^{\min\{\beta,2\}}+\epsilon_{\Hcal_j}\big)
    % \end{align*} 
    \begin{align*}
        \|h_{\hat{\theta}} - h_0\|_2^2 \leq \min_{j \in [M]}O\left(\frac{\delta_{n,j}^2 }{\alpha^2} + \alpha^{\min\{\beta+1,2\}-1} +\frac{\epsilon_{\Hcal_j}^2}{\alpha}  +  \frac{\epsilon_\Gcal}{\alpha^2}. \right)
    \end{align*}
\end{theorem}

We explain its implications. Most importantly, our obtained rate is the best (i.e., oracle rate) among rates when invoking a result of (convergence result for RDIV in \pref{thm: mle-misspec} with misspecified model) for each function class $\Hcal_i$. 
Some astute readers might wonder whether we can just invoke \pref{thm: mle-misspec} by making new function classes $\Hcal_{\text{best}}:=\{h_{\theta}:\theta=e_1,\dots,e_M\}$ or $\Hcal_{\text{conv}}:=\{h_{\theta}:\sum_j \theta_j = 1, \theta_j \geq 0 \}$, and bound the misspecification error $\epsilon_{\Hcal_{\text{conv}}}$ or $\epsilon_{\Hcal_{\text{best}}}$ by $\|h_j - h_0\|$ will lead to a slower rate with an extra factor of $\frac{1}{\alpha}$. The key is only to handle the misspecification error once to avoid the $\frac{1}{\alpha}$ factor by deferring the invocation of strong convexity and working with the excess risk (difference in the expected loss) instead of the $L_2$ difference.

\vspace{-2mm}
\section{Extension to Iterative Version}\label{sec: iterative}
\vspace{-2mm}

One drawback of the result so far is its lack of adaptability to the degree of ill-posedness in the inverse problem, especially for larger values of $\beta$ corresponding to milder problems, when $\beta \geq 2$. To address this issue, in this section, we further generalize our results in Section \ref{sec:method-intro} and \ref{sec: finite-sample}, and propose an iterated Regularized Deep method, which is summarized in Algorithm \ref{alg:dbiv}. In this algorithm, instead of targeting \eqref{eq:pop-reg-noniter}, we target $h_{m, *}$, which is given by the following recursive least square regression with Tikhonov regularization: \begin{align}\label{eq:popu-iter}
    h_{m, *} = &\argmin_{h\in\Hcal} \E[(Y - \Tcal h(Z))^2]+ \alpha\cdot \E[(h - h_{m-1,*})^2(X)].
\end{align}
and we set $h_{-1,*} = 0$. This is the recursive version of the previous regularized objective in Equation \eqref{eq:pop-reg-noniter}, by using Tikhonov regularization around a prior target $h_{m-1,*}$ instead of $0$. Then, with the learned conditional density $\hat{g}$ by MLE in Equation \eqref{eq:mle-est}, we construct an estimator in \eqref{eq: emp-iter} by replacing expectation and an operator $\Tcal$ with empirical approximation and the learned operator $\hat \Tcal$, respectively, in Equation \eqref{eq:popu-iter}. 

\begin{algorithm}[!t]
\caption{Iterative Regularized Deep IV}
\begin{algorithmic}[1]\label{alg:dbiv}
        \REQUIRE Dataset $\{X_i,Y_i,Z_i\}_{i\in[n]}$, function class $\Gcal$, function class $\Hcal$, $\hat{h}_{-1} = 0$
        \STATE Learn $\hat{g}(x|z)$ by MLE \eqref{eq:mle-est}
        \FOR{$m = 1,2,\cdots, M$}
        \STATE   Learn $\hat{h}_m$ by iterative Tikhonov estimator as the following:
        \begin{align}\label{eq: emp-iter}
    \hat{h}_{m}=\argmin_{h\in\Hcal} \E_n[\big(Y - \hat{\Tcal} h(Z)\big)^2]+ \alpha \cdot \E_n[\big(h(X) - \hat{h}_{m-1}(X)\big)^2],
\end{align}
        \ENDFOR
        \OUTPUT $\hat h_M$
    %     \STATE \begin{align}\label{eq: emp-iter}
    % \hat{h}_{m}=&\argmin_{h\in\Hcal} \E_n[\big(Y - \hat{\Tcal} h(Z)\big)^2]\nonumber\\
    % &\qquad \qquad+ \alpha \cdot \E_n[\big(h(X) - \hat{h}_{m-1}(X)\big)^2],
    % \end{align}
\end{algorithmic}
\end{algorithm}

Now, we delve into estimating the finite sample convergence rate of Algorithm \ref{alg:dbiv}. Our findings are summarized in the following theorem.

\vspace{-2mm}
\begin{theorem}[$L_2$ convergence rate for iterative MLE estimator]\label{thm: mle-nonmisspec}
    Suppose Assumption  \ref{ass:source-cond}, \ref{ass: realizability}, \ref{ass:lower-bound} hold. Let $\|Y\|_\infty\leq C_Y$, $\|h\|_\infty \leq C_\Hcal$ holds for all $h\in\Hcal$, $\|g\|_\infty \leq C_\Gcal$ holds for all $g\in\Gcal$. By setting $\alpha = \delta_n^{\frac{2}{2+\min\{\beta, 2m\}}}$, with probability at least $1- c_1 m\exp(c_2n\delta_n^2)$, we have $$
    \| \hat{h}_m - h_0 \|_2^2 =O\big(16^{2m}\cdot\delta_n^{\frac{2\min\{\beta, 2m\}}{2 + \min\{\beta, 2m\}}}\big),
    $$
    here $\delta_{n} $ has the same definition in Theorem \ref{thm: mle-nonmisspec-noniter}.
\end{theorem}
\vspace{-2mm}
% \begin{proof}
%     For a detailed proof, see Appendix \ref{sec: proof-iterative}.
% \end{proof}

Importantly, we can have a rate $O\big(\delta_n^{\frac{2\beta}{2 + \beta}}\big)$ in relatively mild conditions while the previous Theorem~\ref{thm: mle-nonmisspec-noniter} (non-iteratie version) can only allow for $O\big(\delta_n^{\frac{2\min(\beta,2)}{2 + 2\min(\beta,2)}}\big)$, and cannot fully leverage the well-posedness of $h_0$, illustrated by the source condtion $\beta$. Indeed, if we choose the iteration number $m = \lceil \min\{\beta/ 2, \log\log(1/\delta_n)\} \rceil$, then we get a rate of $$\|\hat{h}_m - h_0\|_2^2 = O\bigg(\min\{16^\beta, \log (1/\delta_n)\}\delta_n^{\frac{2\min\{\beta, 2m\}}{2 + \min\{\beta, 2m\}}}\bigg) .$$ Hence for any constant $\beta$, as $n$ grows, eventually $\log\log 1/\delta_n \geq \beta $, and we get the rate of $O\big(\delta_n^{\frac{2\beta}{2 + \beta}}\big)$. This rate can be achieved even if $\beta$ grows with $n$, as long as it grows slower than $O(\log\log 1/\delta_n)$. If $\delta_n = O(n^{-\iota})$ for some $\iota > 0 $, e.g. RKHS or first order Sobolev space \citep[Chapt 14.1.2]{wainwright2019high}, then we note that we can set $m = \lceil \min\{\beta/ 2, \sqrt{\log(1/\delta_n)}\} \rceil$, and $16^{\sqrt{\log(1/\delta_n)}} = O(n^{\epsilon})$ for any $\epsilon > 0$, thus we still obtain a rate of $O(\delta_n^{\frac{2\beta}{2+\beta}})$ when $\sqrt{\log(1/\delta_n)} \geq \beta/2$. In such a case, we can obtain a $O(\delta_n^{\frac{2\beta}{2+\beta}})$ rate even $\beta$ grows with $n$, as long as it grows slower than $\sqrt{\log(1/\delta_n)}$.

Our results for the iterative estimator match the state-of-the-art convergence rate with respect to $L_2$ norm for an iterative estimator in \cite{bennett2023source}. However, their method requires a minimax computation oracle, while our method does not.

% \section{Proof Sketch}
% Due to space limit, in this section, we only include the proof sketch of Theorem \ref{thm: mle-nonmisspec-noniter} and Theorem \ref{thm:model_selection}, i.e. . For detailed proof of results in this paper, please refer to the Appendix.

\section{Numerical Experiments}\label{sec:experiment}

In this section, we evaluate our proposal by numerical simulation.  In particular, we present the performance of RDIV when we use neural networks as the function approximator and the validity of the proposed model selection procedure. We show that with model selection, our method can achieve state-of-the-art performance in a wide range of data-generating processes.  

\subsection{Experimental Settings}

\begin{figure}
    \centering
\begin{tikzpicture}
\node[draw, circle, text centered, minimum size=0.75cm, line width= 1] (x) {$S'$};
\node[draw, circle, above=1 of x,
text centered, minimum size=0.75cm, dashed,line width= 1] (u) {$U$};
\node[draw, circle,left=2.5 of x, minimum size=0.75cm, text centered,line width= 1] (z) {$Q'$};
\node[draw, circle,right=2.5 of x, minimum size=0.75cm, text centered,line width= 1] (w) {$W'$};
\node[draw, circle, below left = 1 and 1 of x, minimum size=0.75cm, text centered,line width= 1] (a) {$A$};
\node[draw, circle, below right = 1 and 1 of x, minimum size=0.75cm, text centered,line width= 1] (y) {$Y$};
\draw[-latex, line width= 1] (x) -- (a);
\draw[-latex, line width= 1] (x) -- (y);
\draw[-latex, line width= 1, dashed] (x) -- (z);
\draw[-latex, line width= 1, dashed] (x) -- (w);
\draw[-latex, line width= 1, dashed] (u) -- (x);
\draw[-latex, line width= 1] (u) -- (z);
\draw[-latex, line width= 1] (u) -- (w);
\draw[-latex, line width= 1] (u) -- (a);
\draw[-latex, line width= 1] (u) -- (y);
\draw[-latex, line width= 1] (a) -- (y);
\draw[-latex, line width= 1, dashed] (z) -- (a);
\draw[-latex, line width= 1, dashed] (w) -- (y);
\end{tikzpicture}
\caption{A typical causal diagram for negative controls. The dashed edges may be absent, and the dashed circle around $S'$ indicates that $U$ is unobserved.}
    \label{fig:causal_dag}
\end{figure}
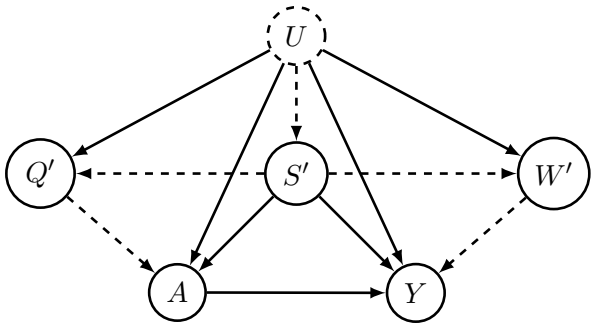

\paragraph{Experiment Design.} In our experiment, we test our method on a synthetic dataset. We adjust the data generating process (DGP) for proximal causal inference used in \cite{cui2020semiparametric,miao2018confounding,deaner2021many}. Concretely, we generate multi-dimensional
variables $U',S', W',Q',A$, where $U$ is an unobserved confounder, $S'\in d_{S}$ is the observed covariate, $W'\in d_{W}$ is the negative control outcomes, $Q'\in d_{Q}$ is the negative control actions, and $A$ is the selected treatment, as described in \pref{fig:causal_dag}. We left the detailed generation process in Appendix \ref{sec:exp-details}.  For a detailed understanding of this setup, we refer the reader to Section 2 of \citet{kallus2021causal}. It is well known that there exists a bridge function $h_0'$ such that the following moment condition holds \citep{cui2020semiparametric,kallus2021causal}: $$
\E[Y - h_0'(W',A,S')|Q',A,S'] = 0,
$$
which allows the concrete form of \eqref{eq:cond-mom-condi}.
To introduce nonlinearity, we transform $(S', W',Q')$ into $(S,W,Q)$ via $S = g(S'), W = g(W'), Q=g(Q')$, where $g(\cdot)$ is a nonlinear invertible function
applied elementwise to $S',W',Q'$ respectively. We consider several forms of $g(\cdot)$, including identity, polynomial, sigmoid design, and exponential function.  In the final data, we only observe $(S,W,Q)$ but not $(S',W',Q')$.  Here we use 6 different $g(\cdot)$: Id$(t) = t$, Poly$(t)$ = $t^3$, LogSigmoid$(t)= \log(1+ |16*x - 8|) \cdot \operatorname{sign}(x)$, Piecewise$(t) = 3(x-2)1_{x\leq1}+ \log(8x-8)1_{x\geq1}$, Sigmoid$(t) = \frac{5}{1+\exp(-0.1*x)}$ and CubicRoot = $x^{1/3}$.

\paragraph{Methods to compare.}

In this experiment, our goal is to estimate the counterfactual mean parameter $\E[Y(1)]$, which is unique as long as \eqref{eq:cond-mom-condi} holds. We learn $h_0$ in \eqref{eq:cond-mom-condi} by RDIV, which corresponds to the procedure in Algorithm \ref{alg:dbiv-noniterative} with MLE for conditional density estimation.   We show results for different values for $\alpha \in \{0.01,0.1\}$, and compare the performance of our approach to that of several different methods, including KernelIV \citep{singh2019kernel}, DeepIV \citep{hartford2017deep}, DeepFeatureIV \citep{xu2021deep}, and AGMM \citep{DikkalaNishanth2020MEoC}. Note that DeepIV can be viewed as a special case of our methods, with $\alpha$ fixed to be 0. In the first stage of our algorithm, we use a three-layer mixture density network \citep{hartford2017deep,rothfuss2019conditional} as the approximator of the conditional density. In the second stage, we use a three-layer fully-connected neural network as the approximators for RDIV, DeepIV, AGMM, and DFIV. We present the results of our method and its comparison with previous benchmarks in terms of MSE normalized by the true estimand value in Table \ref{table:15-1-500}-\ref{table:20-15-1000}. Every estimate is calculated by 100 random replications.  The confidence interval is calculated by 2 times the standard deviation. 

\paragraph{Hyperparameter settings.} For RDIV, we use Adam as the optimizer for both density estimation and Tikhonov regression, with a default learning rate of $10^{-4}$, a batch size of $50$, and a training epoch of $300$. We will show how to choose these hyperparameters with our model selection procedure (Algorithm \ref{alg:model_selection}) in Section~\ref{subsec:model_seleciton}. For all baselines except for AGMM, we adapt the hyperparameters in their original codebase. For AGMM, we tune the learning rate for the learner and adversary for every $g(\cdot)$ independently. We follow \cite{singh2019kernel} to use Gaussian RKHS for function approximation and their method for tuning the regularization parameter. When $n=500$, the learning rate of the learner and adversary in AGMM are manually set to $10^{-4}$ for LogSigmoid, Piecewise, and Sigmoid, and $10^{-3}$ for Id, Poly, and CubicRoot. When $n=1000$, the learning rate of the learner and adversary in AGMM are manually set to $10^{-4}$ for Piecewise and Sigmoid, and $10^{-3}$ for LogSigmoid, Piecewise, and CubicRoot.  The training parameter of DFIV is adopted from \citet{xu2021deep}. Note that tuning DFIV is highly intractable in practice, as their method is essentially a bilevel optimization, which is known to be hard to solve \citep{hong2023two}.

\subsection{Results}

First, we can observe that although our estimator resembles DeepIV, the later fix $\alpha =0$ in \eqref{eq: emp-iter}, RDIV outperforms DeepIV for all $g(\cdot)$.  This is due to the nonzero regularization term, which improves the performance of our estimator by a better tradeoff between bias and variance. Second, in most cases, AGMM and DFIV are outperformed by algorithms that only need single-level optimization (RDIV, KernelIV, DeepIV). This would be because, in these methods, optimization of the loss function is much harder, which results in the inaccuracy of estimators. Thirdly, while it is seen that Kernel IV is comparable to RDIV in some scenarios such as in Table \ref{table:20-15-500} and \ref{table:20-15-1000}, in the next section, we will show that our RDIV equipped with a model selection procedure can generally outperform KernelIV. 

% \masa{I think this tuning-free is not the right way to checaretrize these algorithms. I feel our algorithms also need tuning as well. I would say something like " For most cases, AGMM and DFIV are outperformed by tuning-free algorithms (RDIV, KernelIV, DeepIV. This would be because, in these methods, optimization of loss function is much harder, which results in the inaccuracy of estimators.     }

% \masa{Write interpretation like our method beats the kernel IV, etc. }

\begin{table}[t!]
\centering
\caption{ $\mathbb{E}[Y(1)]$: $d_S = d_Q = 15$, $d_W = 1$, $n_1 = 500$.  }
\label{table:15-1-500}
\resizebox{\textwidth}{!}{%
\begin{tabular}{@{}lcccccr@{}}
\toprule
$g(t)$ & RDIV ($\alpha = 0.01$) & RDIV ($\alpha =0.1$) & KernelIV & DeepIV & DFIV & AGMM \\ \midrule
Id$(t)$  & 0.0077 $\pm$ 0.0012  & \bf{0.0021 $\pm$ 0.0007}  & 0.0193 $\pm$ 0.0018  & 0.0089 $\pm$ 0.0015   & 0.1069 $\pm$ 0.0218  &  0.0198 $\pm$ 0.0011 
 \\
Poly$(t)$ & \bf{0.0150 $\pm$ 0.0057}  & 0.0904 $\pm$ 0.0202  & 0.0439 $\pm$ 0.0062  & 0.0887 $\pm$ 0.0276   & 0.0920 $\pm$ 0.0046   & 0.0453 $\pm$ 0.0023 \\
LogSigmoid$(t)$ & 0.0094 $\pm$ 0.0013  & \bf{ 0.0022 $\pm$ 0.0009 } & 0.0031 $\pm$ 0.0008  & 0.0152 $\pm$ 0.0026   & 0.1444 $\pm$ 0.0080  & 0.0042 $\pm$ 0.0010     \\
Piecewise$(t)$   & 0.0070 $\pm$ 0.0017  & \bf{0.0024 $\pm$ 0.0009}  & 0.0041 $\pm$ 0.0012  & 0.0076 $\pm$ 0.0012  & 0.0150 $\pm$ 0.0026   & 0.0128 $\pm$ 0.0024 
\\ 
Sigmoid$(t) $  & 0.0206 $\pm$ 0.0026  & \bf{0.0021 $\pm$ 0.0006}  & 0.0380 $\pm$ 0.0025  & 0.0278 $\pm$ 0.0025   & 0.1846 $\pm$ 0.0092 & 0.0070 $\pm$ 0.0014   \\
CubicRoot$(t)$ & 0.0095 $\pm$ 0.0014  & \bf{0.0024 $\pm$ 0.0007}  & 0.0511 $\pm$ 0.0039  & 0.0161 $\pm$ 0.0018  & 0.1357 $\pm$ 0.0200  & 0.0536 $\pm$ 0.0021   \\

\bottomrule
\end{tabular}%
}
\end{table}
\begin{table}[t!]
\centering
\caption{ $d_S = d_Q = 15$, $d_W = 1$, $n_1 = 1000$.}
\label{table:15-1-1000}
\resizebox{\textwidth}{!}{%
\begin{tabular}{@{}lcccccr@{}}
\toprule
$g(t)$ & RDIV ($\alpha = 0.01$) & RDIV ($\alpha =0.1$) & KernelIV & DeepIV & DFIV & AGMM \\ \midrule
Id$(t)$ & 0.0106 $\pm$ 0.0013  & \bf{0.0014 $\pm$ 0.0003}  & 0.0145 $\pm$ 0.0013  & 0.0128 $\pm$ 0.0015  & 0.1162 $\pm$ 0.0052  & 0.0217 $\pm$ 0.0135 \\
Poly$(t)$ & 0.0164 $\pm$ 0.0020  & \bf{0.0037 $\pm$ 0.0027}  & 0.0396 $\pm$ 0.0038  & 0.0182 $\pm$ 0.0023  & 0.1256 $\pm$ 0.0044 & 0.0054 $\pm$ 0.0031 
 \\
LogSigmoid$(t)$ & 0.0078 $\pm$ 0.0009  & \bf{0.0009 $\pm$ 0.0003}  & 0.0259 $\pm$ 0.0023  & 0.0262 $\pm$ 0.0023  & 0.1618 $\pm$ 0.0482  & 0.0053 $\pm$ 0.0010  \\
Piecewise $(t)$ & \bf{0.0017 $\pm$ 0.0004}  & 0.0059 $\pm$ 0.0008  & 0.0080 $\pm$ 0.0008  & {0.0019 $\pm$ 0.0005}  & 0.1623 $\pm$ 0.0674 & 0.0014 $\pm$ 0.0011 \\ 
Sigmoid$(t)$ & \bf{0.0077 $\pm$ 0.0016}  & {0.0082 $\pm$ 0.0023}  & 0.0311 $\pm$ 0.0014  & 0.0110 $\pm$ 0.0019  & 0.2085 $\pm$ 0.0443& 0.0296 $\pm$ 0.0023  \\
CubicRoot$(t)$ & 0.0254 $\pm$ 0.0021  & \bf{0.0048 $\pm$ 0.0008}  & 0.0459 $\pm$ 0.0024  & 0.0248 $\pm$ 0.0022  & 0.1401 $\pm$ 0.0047  & 0.0650 $\pm$ 0.0035  \\

\bottomrule
\end{tabular}%
}
\end{table}
\begin{table}[t!]
\centering
\caption{ $d_S = d_Q = 20$, $d_W = 10$, $n_1 = 500$.}
\label{table:20-15-500}
\resizebox{\textwidth}{!}{%
\begin{tabular}{@{}lcccccr@{}}
\toprule
$g(t)$ & RDIV ($\alpha = 0.01$) & RDIV ($\alpha =0.1$) & KernelIV & DeepIV & DFIV & AGMM \\ \midrule
Id$(t)$  & 0.0272 $\pm$ 0.0022  & \bf{0.0055 $\pm$ 0.0009}  & 0.0088 $\pm$ 0.0016  & 0.0364 $\pm$ 0.0025   & 0.0291 $\pm$ 0.0060  & 0.3291 $\pm$ 0.0115 
 \\
Poly$(t)$ & \bf{0.0067 $\pm$ 0.0016}  & 0.0230 $\pm$ 0.0051  & 0.0697 $\pm$ 0.0041  & 0.0263 $\pm$ 0.0050  & 0.0997 $\pm$ 0.0046  & 0.0409 $\pm$ 0.0225    \\
LogSigmoid$(t)$ & 0.0905 $\pm$ 0.0058  & 0.0525 $\pm$ 0.0054  & 0.0335 $\pm$ 0.0014  & 0.0960 $\pm$ 0.0066  & 0.2059 $\pm$ 0.0826  & \bf{0.0218 $\pm$ 0.0027 }
     \\
Piecewise$(t)$   & 0.0305 $\pm$ 0.0043  & \bf{0.0104 $\pm$ 0.0021}  & 0.0359 $\pm$ 0.0010  & 0.0225 $\pm$ 0.0031  & 0.7626 $\pm$ 0.9996  & 0.0136 $\pm$ 0.0010 
\\ 
Sigmoid$(t) $  & 0.1481 $\pm$ 0.0083  & {0.0106 $\pm$ 0.0028}  & \bf{0.0018 $\pm$ 0.0004}  & 0.1983 $\pm$ 0.0117   & 0.3545 $\pm$ 0.0494  &0.0307 $\pm$ 0.0195    \\
CubicRoot$(t)$ & 0.0810 $\pm$ 0.0039  & 0.0288 $\pm$ 0.0025  &\bf{ 0.0021 $\pm$ 0.0004 } & 0.0949 $\pm$ 0.0050  & 0.0956 $\pm$ 0.0453  & 0.3461 $\pm$ 0.0121    \\

\bottomrule
\end{tabular}%
}
\end{table}
\begin{table}[t!]
\centering
\caption{ $d_S = d_Q = 20$, $d_W = 10$, $n_1 = 1000$.}
\label{table:20-15-1000}
\resizebox{\textwidth}{!}{%
\begin{tabular}{@{}lcccccc@{}}
\toprule
$g(t)$ & RDIV ($\alpha = 0.01$) & RDIV ($\alpha =0.1$) & KernelIV & DeepIV & DFIV & AGMM \\ \midrule
Id$(t)$  & 0.0652 $\pm$ 0.0035  & 0.0269 $\pm$ 0.0020  & \bf{0.0009 $\pm$ 0.0002}  & 0.0639 $\pm$ 0.0033   & 0.1442 $\pm$ 0.2461  & 0.1321 $\pm$ 0.0029 
 \\
Poly$(t)$ & 0.0861 $\pm$ 0.0076  & \bf{0.0224 $\pm$ 0.0034}  & 0.0465 $\pm$ 0.0021  & 0.1148 $\pm$ 0.0082  & 0.0951 $\pm$ 0.0031  & 0.1796 $\pm$ 0.0023    \\
LogSigmoid$(t)$ & 0.0649 $\pm$ 0.0046  & 0.0280 $\pm$ 0.0025  & \bf{0.0197 $\pm$ 0.0014}  & 0.0759 $\pm$ 0.0045  & 0.2949 $\pm$ 0.2917  & 0.0247 $\pm$ 0.0013 \\
Piecewise$(t)$   & {0.0039 $\pm$ 0.0008}  & \bf{0.0037 $\pm$ 0.0006}  & 0.0215 $\pm$ 0.0006  & 0.0065 $\pm$ 0.0012   & 0.5442 $\pm$ 0.4784  & 0.0133 $\pm$ 0.0009 
\\ 
Sigmoid$(t) $   & 0.1112 $\pm$ 0.0053  & 0.0091 $\pm$ 0.0028  & \bf{0.0037 $\pm$ 0.0005}  & 0.1493 $\pm$ 0.0058   & 0.3332 $\pm$ 0.0652  & 0.0650 $\pm$ 0.0029   \\
CubicRoot$(t)$ & 0.0990 $\pm$ 0.0042  & 0.0802 $\pm$ 0.0046  &\bf{ 0.0021 $\pm$ 0.0004 } & 0.1070 $\pm$ 0.0043   & 0.0956 $\pm$ 0.0453  & 0.3461 $\pm$ 0.0121    \\

\bottomrule
\end{tabular}%
}
\end{table}

\begin{table}[t!]
\centering
\caption{ Model selection results based on Best ERM. The left tabular is generated from a data size of $n_1 = 500$, while the right tabular is generated from a dataset with $n_1 = 1000$. Both datasets satisfies $d_S = d_Q = 20, d_W = 10$.}
\resizebox{\textwidth}{!}{
\label{table:model-sel-erm}
\begin{tabular}{lccc|ccc}
\toprule
$g(t)$ & RDIV ($\alpha = 0.01$) & RDIV ($\alpha =0.1$) & KernelIV  & RDIV ($\alpha = 0.01$) & RDIV ($\alpha =0.1$) & KernelIV  \\ \midrule
Id$(t)$  & \bf{0.0017 $\pm$ 0.0017}  & 0.0047 $\pm$ 0.0021  &0.0088 $\pm$ 0.0016  &0.0102 $\pm$ 0.0028 & 0.0014 $\pm$ 0.0009 &\bf{0.0009 $\pm$ 0.0002}
 \\
Poly$(t)$ & \bf{0.0032 $\pm$ 0.0024}   & 0.0272 $\pm$ 0.0097 &0.0697 $\pm$ 0.0041  &0.0313 $\pm$ 0.0137 &\bf{0.0049 $\pm$ 0.0026}    & 0.0465 $\pm$ 0.0021\\
LogSigmoid$(t)$ & 0.0121 $\pm$ 0.0055   & \bf{0.0019 $\pm$ 0.0007} & 0.0335 $\pm$ 0.0014  & 0.0078 $\pm$ 0.0020 &\bf{0.0008 $\pm$ 0.0004} & 0.0197 $\pm$ 0.0014\\
Piecewise$(t)$   & 0.0159 $\pm$ 0.0121   & \bf{0.0020 $\pm$ 0.0019} & 0.0359 $\pm$ 0.0010  &\bf{0.0024 $\pm$ 0.0013}  & 0.0034 $\pm$ 0.0027&0.0215 $\pm$ 0.0006
\\ 
Sigmoid$(t) $  & 0.1655 $\pm$ 0.0144   & 0.0937 $\pm$ 0.0174& \bf{0.0018 $\pm$ 0.0004}  &0.1538 $\pm$ 0.0078 &  0.0863 $\pm$ 0.0187&  \bf{0.0037 $\pm$ 0.0005} \\
CubicRoot$(t)$ & 0.0034 $\pm$ 0.0017   & \bf{0.0019 $\pm$ 0.0021}& 0.0021 $\pm$ 0.0004  &0.0148 $\pm$ 0.0048 &  0.0036 $\pm$ 0.0035  & \bf{0.0021 $\pm$ 0.0004} \\

\bottomrule
\end{tabular}%
}
\end{table}
\subsection{Model selection}
\label{subsec:model_seleciton}

We also report our results in model selection for the second stage by implementing Best-ERM in Algorithm \ref{alg:model_selection} and demonstrate how it improves our results. 
Specifically, our models $h_1, \dots, h_M$ are trained by different hyperparameters. First, we employ model selection for the density function by Best ERM. Then with the trained density function in the first stage, we further apply Best ERM to the models in the second stage.  In the model selection experiments, we fix the dimension of our dataset to be $d_S = d_Q = 20$, $d_W = 10$. We compute the mean and confidence interval with 10 independent trials. We set the candidate training parameters as follows: the number of epochs $\in\{300,400\}$, the batch size for the 1st stage $\in\{30,50\}$ and the batch size for the 2nd stage $\in\{50,60,100\}$, the learning rate $\in\{10^{-4},10^{-3}\}$, the number of mixture components $\in\{40,50,60\}$. As shown in Table \ref{table:model-sel-erm}, when RDIV is equipped with model selection techniques, our method outperforms KernelIV in all but one case when the dataset size is 500, and outperforms KernelIV in 3 out of 6 settings when the dataset size is 1000. Our approach demonstrates its effectiveness by outperforming previous benchmarks across a diverse set of Data Generating Processes (DGP). This achievement is attributed to both the ease of optimization of RDIV and its theoretically sound integration with model selection procedures.

\section{Conclusion}
In this paper, we study NPIV regression with general function approximation. We propose a new estimator defined by the loss that organically combines Tikhonov regularization and an MLE estimator, namely the Regularized DeepIV (RDIV).  We show that our estimator converges to the least norm solution, and derive its convergence rate. Notably, our method does not rely on uniqueness or minimax computation oracle. We further illustrate that our method can be incorporated into model selection and show that our procedure can achieve the oracle rate with respect to the minimal model misspecification error. When extended to an iterative estimator, our method achieves the state-of-the-art convergence rate. Moreover, we justify our method through numerical simulations. Our experiments show that RDIV outperforms existing benchmarks in a wide range of circumstances.
\newpage
\bibliography{Note/rl}
\newpage
\appendix

\section{Additional Related Works for Model Selection}\label{sec:related-works}
\paragraph{Model Selection.}Under the classical supervised learning setting, a common approach is to perform empirical risk minimization (ERM) on a separate validation set, and choose the candidate model that achieves the smallest risk \citep{modelSelection}, or similarly, through M-fold cross-validation which splits the data into M folds, and evaluates the risk on the different held out set for each model \citep{crossVal}. As an alternative to selecting a single model, convex aggregation or linear aggregation is employed to find the best convex/linear combination of models \citep{convexAgg, linearAgg}. However, it can be shown that the aforementioned approaches are sub-optimal in the sense that they cannot achieve the optimal $\frac{\log(M)}{n}$ rate for the model selection residual. To tackle this challenge, \cite{lecue2009aggregation} proposed a different approach for convex aggregation by first finding a subset of "almost minimizers" - a subset of the candidate functions that is sufficiently close to the minimizer within the candidates on the validation set, and then finding a best aggregate in the convex hull of this subset. This approach achieves the optimal model selection rates as it performs ERM on a subset that is much smaller than the convex hull of all candidate models, thereby reducing the statistical error. Furthermore, other optimal model selection approaches include the Q-aggregation approach which performs ERM with a modified loss that adds an additional penalty based on individual model performance \citep{lecue2014optimal}. 

\section{Results when Using $\chi^2$-MLE}\label{sec:chi2}
In this section, we consider another density estimation for the density estimation, the $\chi^2$-MLE: 
 \begin{align}\label{eq:chi2-mle-est}
\hat{g} = \argmin_{g\in\Gcal} 0.5\cdot \E_n\bigg[\int_{\Xcal} g^2(x|Z)d\mu(x)\bigg] - \E_n[g(X|Z)]. 
\end{align}

\subsection{Finite Sample Results}

Although Assumption \ref{ass:lower-bound} is widely accepted in previous works, in practice, it often fails to hold when $g_0$ does not have full support on $\Xcal$. To address this drawback of MLE, in this subsection, we further discuss the finite sample convergence rate of Algorithm \ref{alg:dbiv} when the conditional density estimation is performed by $\chi^2$-MLE. In this case, the first step estimation procedure is given by Equation \eqref{eq:chi2-mle-est}. Notably, our guarantee does not relate to the lower bound of $g_0(x|z)$. Our results rely on the following assumption, which characterizes the smoothness of function class $\Hcal$.
\begin{assumption}[$\gamma$-Smoothness]\label{ass:gamma-smooth}
    For all $h - h' \in \Hcal - \Hcal$, we assume that $\|h - h'\|_\infty \leq \|h - h'\|_2^\gamma$.
\end{assumption}
Such a relationship is known for instance to hold for Sobolev spaces and more generally
for reproducing kernel Hilbert spaces (RKHS) with a polynomial eigendecay. A notable instance is RKHS with eigendevay at a rate of $O(1/j^{1/p})$ for some $p\in(0,1)$. In that case, Lemma 5.1 of \cite{mendelson2010regularization} shows that $\gamma = 1- p$.   For the Gaussian
kernel, which has an exponential eigendecay, we can take $ p $ arbitrarily close to $0$.
We now summarize our result for $\chi^2$-MLE in the following theorem.
\begin{theorem}[$L_2$ convergence rate for RMIV with $\chi^2$-MLE]\label{thm:chi2-mle-result-noniter}
    Suppose Assumption \ref{ass:source-cond},\ref{ass: realizability},\ref{ass:gamma-smooth} hold. By setting $\alpha = \delta_n^{\frac{2}{2+ (2-\gamma)\min\{\beta, 2\}}}$, with probability at least $1- c_1\exp(c_2n\delta_n^2)$, we have $$
    \|\hat{h} - h_0\|_2^2 \leq O\big(\delta_n^{\frac{2\min\{\beta, 2\}}{2+(2-\gamma)\min\{\beta, 2\}}}\big).
    $$
 Here $\delta_{n}$ has the same definition in Theorem \ref{thm: mle-nonmisspec-noniter}.
\end{theorem}
% \begin{proof}
%     For a detailed proof, see Appendix \ref{sec:proof-mle-nonmisspec-noniter}.
% \end{proof}
The convergence rate of RMIV with $\chi^2$-MLE depends on the smoothness parameter $\gamma$. As $\gamma \rightarrow 1$ ,  we have $\|\hat{h} - h_0\|_2^2 \leq O\big(\delta_n^{\frac{2\min\{\beta,2\}}{2+\min\{\beta,2\}}}\big)$, which recovers the rate in Theorem \ref{thm: mle-nonmisspec}.  We further discuss the results for $\chi^2$-MLE based IV regression under misspecification. 
\begin{theorem}[$L_2$ convergence rate for RMIV with $\chi^2$-MLE under misspecification]\label{thm: chi2mle-misspec}
    Suppose Assumption \ref{ass:source-cond},\ref{ass:gamma-smooth} hold, and there exists $h^{\dag}\in \Hcal$ and $g^{\dag} \in \Gcal$ such that $\|h_0 - h^{\dag}\|_2 \leq \epsilon_\Hcal$ and $\E\big[\int_{\Xcal} (g^{\dag}(x|Z) - g_0(x|Z)  )^2 d \mu(x)\big] \leq \epsilon_\Gcal$. For any $0<\alpha\leq 1$,with probability at least $1- c_1\exp(c_2n\delta_n^2)$, we have \begin{align*}
    &\|\hat{h} - h_0\|_2^2\leq O\bigg( \bigg(\frac{\delta_n^2+\epsilon_\Gcal}{\alpha^2} \bigg)^{1/(2-\gamma)} + \alpha^{\min\{\beta+1,2\}-1}+\frac{\epsilon_\Hcal^2}{\alpha}\bigg),
    \end{align*}
    Here $\delta_n$ has the same definition in Theorem \ref{thm: mle-nonmisspec-noniter}. 
\end{theorem}
% \begin{proof}
%     For a detailed proof, see Appendix \ref{sec: proof-misspec}.
% \end{proof}
\begin{remark}
    We define $\epsilon:= \{\epsilon_\Gcal,\epsilon_\Hcal^2 \}$. If $\epsilon < 1$, then by setting $\alpha = (\delta_n^2 + \epsilon)^{\frac{2}{2+(2-\gamma)\min\{\beta,1\}}}$, we have $$
    \|\hat{h} - h_0\|_2^2 \leq O\big((\delta_n^2 + \epsilon)^{\frac{2\min\{\beta,1\}}{2+(2-\gamma)\min\{\beta,1\}}}\big).
    $$
    If $\epsilon \geq 1$, then by setting $\alpha = 1$, we have $\|\hat{h} - h_0\|_2^2 \leq O(\epsilon^{1/(2-\gamma)})$.
    
\end{remark}
\subsection{Results for Model Selection}

\pref{thm:model_selection} is extended when using $\chi^2$-MLE. Indeed, if Assumption \ref{ass:gamma-smooth} holds and the candidate function are trained with $\hat{g}$ estimated using the $\chi^2$-MLE approach, the output of Convex-ERM or Best-ERM $\hat{\theta}$, satisfies
    \begin{align*}
    \|h_{\hat{\theta}} - h_0\|_2^2 \leq \min_{j}O\left(\alpha^{\min\{\beta+1,2\}-1} + \left(\frac{\delta_{n,j}^2+\epsilon_\Gcal}{\alpha^2} \right)^{1/(2-\gamma)} + \frac{1}{\alpha} \epsilon_{\Hcal_j}^2\right). 
\end{align*}

\subsection{Convergence Results for Iterative Version}

We further discuss the finite sample convergence rate of Algorithm \ref{alg:dbiv} when the conditional density estimation is performed by $\chi^2$-MLE. In this case, the first step estimation procedure is given by Equation \eqref{eq:chi2-mle-est}. Notably, in this case, we do not require the ground truth density $g_0$ to be uniformly lower bounded, which is assumed in Assumption \ref{ass:lower-bound} and serves as a prerequisite for MLE convergence. Our results are summarized by the following theorem. 
\begin{theorem}[$L_2$ convergence rate for iterative $\chi^2$-MLE estimator]\label{thm:chi2-mle-result}
    Under Assumption \ref{ass: sol-exist},\ref{ass:source-cond},\ref{ass: realizability},\ref{ass:gamma-smooth}, by setting $\alpha = \delta_n^{\frac{2}{2+ (2-\gamma)\min\{\beta, 2m\}}}$, with probability at least $1- c_1 m\exp(c_2n\delta_n^2)$, we have $$
    \|\hat{h}_m - h_0\|_2 \leq O\big(16^{2m}\cdot \delta_n^{\frac{2\min\{\beta, 2m\}}{2+(2-\gamma)\min\{\beta, 2m\}}}\big).
    $$
    Here $\delta_n$ has the same definition in Theorem \ref{thm: mle-nonmisspec-noniter}.
\end{theorem}
% \begin{proof}
%     For a detailed proof, see Appendix \ref{sec: proof-iterative}.
% \end{proof}
\begin{remark}
Similar to Section \ref{sec:misspecified}, by setting the iteration number $m = \lceil \min\{\beta/ 2, \log\log(1/\delta_n)\} \rceil$, we have $$
   \|\hat{h}_m - h_0\|_2 \leq O\bigg(16^{2m}\cdot \delta_n^{\frac{2\min\{\beta, 2m\}}{2+(2-\gamma)\min\{\beta, 2m\}}}\bigg).
$$
Therefore, for $\log\log \delta_n \geq \beta$, eventually we have the rate of $O\big(\delta_n^{\frac{2\beta}{2+ (2-\gamma) \beta}}\big)$. If $\delta_n = O(n^{-\iota})$, then we can set $m = \lceil \min\{\beta/ 2, \sqrt{\log(1/\delta_n)}\} \rceil$ to obtain the same rate.
Moreover, if $\gamma \rightarrow 1$, e.g.  RKHS with exponential eigenvalue decay \citep[Lemma 5.1]{mendelson2010regularization}, then we recover the rate of $O\big(\delta_n^{\frac{2\beta}{2 + \beta}}\big)$ even without Assumption \ref{ass:lower-bound}.
\end{remark}

\section{Proof of Theorem \ref{thm: mle-nonmisspec-noniter} and \ref{thm:chi2-mle-result-noniter}}\label{sec:proof-mle-nonmisspec-noniter}
In this section, we prove the convergence rate of non-iterative RMIV. We prove the results of Theorem \ref{thm: mle-nonmisspec-noniter} and \ref{thm:chi2-mle-result-noniter}\label{sec:proof-mle-nonmisspec-noniter} respectively. Recall that we define \begin{align}
h_{*} := \argmin_{h\in \Hcal} \|Y-\Tcal h \|_2^2 + \alpha \|h\|_2^2,
\end{align}
by Lemma \ref{lem:iter-tikh-bias}, we have $$
\|h_{*} - h_0\|_2^2 \leq \|w_0\|^2 \alpha^{\min\{\beta,2\}}.
$$
Therefore, we only need to provide an upper bound for $\|\hat{h} - h_{*}\|_2^2$. We start by proving the following lemma, and with the convergence rate of MLE and $\chi^2$-MLE, we conclude the proof of Theorem \ref{thm: mle-nonmisspec-noniter} and Theorem \ref{thm:chi2-mle-result-noniter} respectively. 
\begin{lemma}\label{lem: quadra-lemma}
With probability at least $1 - c_1\exp(c_2 n\delta_{n,\Hcal}^2)$, we have the following inequality:
    \begin{align*}
    &\alpha \|\hat h - h_* \|^2_2 + \|\Tcal(\hat h-h_*)\|^2_2 \leq \EE[\ell_{\hat{h},g_0} -  \ell_{h_*,g_0}] \\
    &\qquad\qquad= O\bigg( \delta_{n,\Hcal}\{ (\alpha+1) \|\hat h-h_*\|_2  +   \delta_{n,\Hcal}   \}+  \|(\hat \Tcal-\Tcal)(\hat h-h_*)\|_1\bigg).
    \end{align*}
\end{lemma}
\begin{proof}
By the optimality of $h_*$ in Eq.~\eqref{eq:pop-reg-noniter}, we have 
\begin{align*}
    & \alpha \|\hat{h} - h_* \|^2_2 + \|\Tcal(\hat{h}-h_*)\|^2_2  \leq \EE[L(\Tcal \hat{h}) ] - \EE[L (\Tcal h_*)] + \alpha\{\EE[\hat{h}^2(X)] - \EE[h_*(X)^2]\},
\end{align*}
where define $L(\Tcal h):= (Y- \Tcal h)^2$. 
Recall that \begin{align*}
    &\EE[L(\Tcal \hat{h}) ] - \EE[L (\Tcal h_*)] + \alpha\{\EE[\hat{h}^2(X)] - \EE[h_*(X)^2]\} =\\
    &\qquad\EE[ - 2 Y  \Tcal (\hat{h}-h_*)(Z)  + (\Tcal \hat{h})^2(Z) -  (\Tcal h_*)^2(Z) ] +  \alpha\{\EE[\hat{h}^2(X)] - \EE[h_*(X)^2]\},
\end{align*}

we have 
\begin{align}\label{eq:single-quadratic}
    & \alpha \|\hat{h}- h_* \|^2_2 + \|\Tcal(\hat{h}-h_*)\|^2_2  \nonumber\\ 
    &\qquad= \EE[ - 2 Y  \Tcal (\hat{h}-h_*)(Z)  + (\Tcal \hat{h})^2(Z) -  (\Tcal h_*)^2(Z) ] +  \alpha\{\EE[\hat{h}^2(Z)] - \EE[h_*(Z)^2]\}\nonumber\\
    &\qquad= \EE[ - 2 Y  \hat \Tcal (\hat{h}-h_*)(Z) + (\hat \Tcal \hat{h})^2(Z) - (\hat \Tcal h_*)^2(Z) ]   +  C_1 \times \EE[|(\hat \Tcal - \Tcal) (\hat{h}-h_*)(X)|] \nonumber\\ 
    &\qquad\qquad+  \alpha\{\EE[\hat{h}^2(X)] - \EE[h_*(X)^2]\} \nonumber\\
    &\qquad\leq \mathrm{Emp} + \mathrm{Loss} + C_1 \times \EE[|(\hat \Tcal - \Tcal) (\hat{h}-h_*)(Z)|]  \nonumber\\ 
    &\qquad= \mathrm{Emp} + \mathrm{Loss} + \|(\hat \Tcal-\Tcal) (  \hat{h}- h_*)\|_1,
\end{align}
here the inequality comes from the uniform boundedness of $\hat{h}, h_*, \Tcal{h}, \Tcal{\hat{h}},\hat{\Tcal}h,\hat{\Tcal}\hat{h} $,  and the $O(1)$-Lipschitz of $L(\cdot)$.
\begin{align*}
    &\mathrm{Emp} = |(\EE_n-\EE)  [L(\hat \Tcal \hat{h})- L(\hat \Tcal h_*)+\alpha(\hat{h}^2(X)-h_*(X)^2)]|,  \\ 
    & \mathrm{Loss} =  \EE_n[ -2Y  \hat \Tcal (\hat{h}-h_*)(Z)  + (\hat \Tcal \hat{h})^2(Z) - (\hat \Tcal h_*)^2(Z) + \alpha\{\hat{h}^2(X)- h_*(X)^2 \}]. 
\end{align*} 
Here, using Lemma \ref{lem:local-concen}, the term $\mathrm{Emp} $ is upper-bounded as follows with probability at least $1 - c_1\exp(c_2 n\delta_{n,\Hcal}^2)$: 
\begin{align}\label{eq:single-emp}    
  \mathrm{Emp}  &\leq \delta_{n,\Hcal}\{ \alpha \|\hat{h}-h_*\|_2+ \|\hat \Tcal(\hat{h}-h_*)\|_2+   \delta_{n,\Hcal}   \} \nonumber\\ 
&\leq \delta_{n,\Hcal}\{ \alpha \|\hat{h}-h_*\|_2 +  \|\hat{h}-h_*\|_2 +   \delta_{n,\Hcal}   \}. 
\end{align} 
% In the first inequality, we use the following variance bound: 
% \begin{align*}
%     &\EE[\{L(\hat \Tcal \hat{h})- L(\hat \Tcal h_*)\}^2]\\
%     &\leq \EE\left [ \left \{ -2 Y ( \hat \Tcal f - \hat \Tcal h_*)(Z) + \{\hat \Tcal \hat{h}-\hat \Tcal h_* \}(Z)\{ \hat \Tcal \hat{h} + \hat \Tcal h_*\}(Z) \right \}^2 \right ] \\ 
%     &\leq C_2 \times \|\hat \Tcal (  \hat{h} - h_*)\|^2_2. 
% \end{align*}

Furthermore, recall that by our iteration in \eqref{eq:emp-reg-noniter}, we have
\begin{align*}
     \EE_n[  - 2 Y  \Tcal (\hat{h}-h_*)(Z)  + (\Tcal \hat h)^2(Z) - (\Tcal h_*)^2(Z)+ \alpha\{\hat h^2(X)- h_*(X)^2 \} ]\leq 0. 
\end{align*}
Hence, we have 
\begin{align}\label{eq:single-loss}
    \mathrm{Loss}\leq 0. 
\end{align}

Combining everything, we have 
\begin{align}\label{eq: useful}
    \alpha \|\hat h - h_* \|^2_2 + \|\Tcal(\hat h-h_*)\|^2_2 \leq \delta_{n,\Hcal}\{ (\alpha+1) \|\hat h-h_*\|_2 +   \delta_{n,\Hcal}   \}+  \|(\hat \Tcal-\Tcal)(\hat h-h_*)\|_1, 
\end{align}
Here the constant $c_1$ and $c_2$ hide constants related to $C, C_0$. The first inequality comes from \eqref{eq:single-emp}. We implicitly use $ \alpha \leq 1$ in the last inequality. 
\end{proof}
\paragraph{Proof of Theorem \ref{thm: mle-nonmisspec-noniter}.} By Assumption \ref{ass: realizability}, we have $\epsilon_\Gcal = 0$. By Corollary \ref{cor: T-mle} and Lemma \ref{lem: quadra-lemma},  since $\alpha\leq 1$ we have \begin{align*}
      \alpha \|\hat h - h_* \|^2_2 + \|\Tcal(\hat h-h_*)\|^2_2&= O\bigg(\delta_{n,\Hcal}\{ (\alpha+1) \|\hat h-h_*\|_{2} +   \delta_{n,\Hcal}   \}+  \delta_{n,\Gcal} \|\hat{h} - h_*\|_2\bigg) \nonumber\\
     &\leq c_1 \delta^2_n  + c_2 \delta_n\|\hat{h} - h_*\|_2 \tag{$\delta_n := \max\{\delta_{n,\Gcal}, \delta_{n,\Hcal}\}$}\\
     &\leq c_1\delta_n^2 +2c_2''\delta_n^2/\alpha + 2c_2'\alpha\|\hat{h} - h^*\|_2^2 \tag{$2ab \leq ca^2 + \frac{b^2}{c}$}
\end{align*}
holds with probability at least $1 - c\exp(n\delta_n^2)$, where $c_2'\leq 1$.
By Lemma \ref{lem:aux-1}, we have 
\begin{align*}
\|\hat h - h_* \|^2_2 \leq  O( (\delta_n^2 / \alpha^2)  + \delta_n^2 / \alpha \big) = O(\delta_{n}^2 / \alpha^2), \tag{$\alpha \leq 1$}
\end{align*}
therefore by Lemma \ref{lem:iter-tikh-bias}, we have $$
    \|\hat h-h_0 \|^2_2\leq  \delta_{n}^2 / \alpha^2+ \alpha^{\min(\beta,2)} ,
$$
set $\alpha = \delta_n^{\frac{2}{2+\min\{\beta, 2\}}}$, and we conclude the proof of Theorem \ref{thm: mle-nonmisspec-noniter}.
\paragraph{Proof of Theorem \ref{thm:chi2-mle-result-noniter}.}
By Assumption \ref{ass: realizability}, we have $\epsilon_\Gcal = 0$. By Corollary \ref{cor: T-mle} and Lemma \ref{lem: quadra-lemma}, we have \begin{align*}
      \alpha \|\hat h - h_* \|^2_2 + \|\Tcal(\hat h-h_*)\|^2_2&\leq \delta_{n,\Hcal}\{ \alpha \|\hat h-h_*\|_{2} + \|\Tcal(h-h_*)\|_2 +   \delta_{n,\Gcal} \|\hat{h} - h_*\|_\infty+   \delta_{n,\Hcal}  \}+  \delta_{n,\Gcal} \|\hat{h} - h_*\|_\infty ,
\end{align*}
By Assumption \ref{ass:gamma-smooth}, we have \begin{align*}
\alpha \|\hat h - h_* \|^2_2 + \|\Tcal(\hat h-h_*)\|^2_2 &\leq \delta_{n,\Hcal}\{ (\alpha +1)\|\hat h-h_*\|_{2} +   \delta_{n,\Hcal}   \}+  \delta_{n,\Gcal} \|\hat{h} - h_*\|_2^\gamma\\
&\leq c_1 \delta_n\|\hat{h} - h_*\|_2 +c_2 \delta_n \|\hat{h} - h_*\|_2^\gamma \tag{$\delta_n := \{\delta_{n,\Gcal}, \delta_{n,\Hcal}\}$},
\end{align*}
By Lemma \ref{lem:aux-1}, we have $$
\|\hat{h} - h_{*}\|_2^2 \leq O((\delta_n/\alpha)^{\frac{2}{2-\gamma}} + (\delta_n/\alpha)^2)  \leq O(\delta_n/\alpha)^{\frac{2}{2-\gamma}}
$$
since $\gamma \in (0,1)$. Therefore, by Lemma \ref{lem:iter-tikh-bias}, we have \begin{align*}
    \|\hat h-h_0 \|^2_2\leq  (\delta_n/\alpha)^{\frac{2}{2-\gamma}}+ \alpha^{\min(\beta,2)}.
\end{align*}
By selecting $\alpha = O(\delta_n^{\frac{2}{2+ (2 - \gamma)\min\{\beta,2\}}})$, we have $$
\|\hat h-h_0 \|^2_2\leq \delta_n^{\frac{2\min\{\beta,2\}}{2+ (2 - \gamma)\min\{\beta,2\}}},
$$
and we conclude the proof of Theorem \ref{thm:chi2-mle-result-noniter}.
\section{Proof of Theorem \ref{thm: mle-misspec} and \ref{thm: chi2mle-misspec}}\label{sec: proof-misspec}
In this section, we consider the case when $\epsilon_\Gcal$ and $\epsilon_{\Hcal}$ doht equal zero, i.e. Assumption \ref{ass: realizability} does not hold. We aim to establish a convergence rate for $\|\hat{h} - h_0\|_2$ for both MLE-based RDIV and $\chi^2$-MLE based RDIV in terms of $\delta_n$, $\epsilon_\Hcal$ and $\epsilon_\Gcal$. 
\begin{lemma}\label{lem:misspec-main}
Under Assumption \ref{ass:source-cond}, for $\alpha \in (0,1)$ we have 
    \begin{align*}
    \|\hat{h} - h_0\|^2 \leq 3\|\hat{h} - h_{*}\|^2+O\bigg(\frac{1}{\alpha}\bigg\{\epsilon_\Hcal^2 + \alpha^{\min\{\beta+1,2\}}\bigg\}\bigg),
\end{align*}
\end{lemma}
\begin{proof}
Note that in the misspecified case, we no longer have $h_0\in\Hcal$. We further a augmented function class $\Hcal' = \operatorname{Span}(\Hcal \cup \{h_0\})$, and the corresponding optimizer of $\Lcal_0$ on $\Hcal$ and $\Hcal'$: 
\begin{align*}
&h_*' = \argmin_{h\in\Hcal'}\|\Tcal(h - h_0)\|_2^2 + \alpha \|h\|^2,\\
&h_* = \argmin_{h\in\Hcal} \|\Tcal(h - h_0)\|_2^2 + \alpha \|h\|^2.
\end{align*}
We define a function $$
\Lcal_0(t):= \|\Tcal(h_*' + t(h_* - h_*') - h_0)\|_2^2 + \alpha \|h_*' + t(h_* - h_*')\|^2,
$$
then $\Lcal_0$ is $\alpha$-strongly convex, and attains its minimum at $\Lcal_0(0)$. Note that we have the following inequality holds for all $h\in\Hcal$, \begin{align*}
\frac{1}{\alpha}(L_0(1) - L_0(0))&=\frac{1}{\alpha}\bigg\{ \|\Tcal(h_* - h_0)\|^2 + \alpha \|h_*\|^2 - (\|\Tcal(h_*' - h_0)\|^2 + \alpha \|h_*'\|^2) \bigg\} \\
&\leq \frac{1}{\alpha}\bigg\{ \|\Tcal(h - h_0)\|^2 + \alpha \|h\|^2 - (\|\Tcal(h_*' - h_0)\|^2 + \alpha \|h_*'\|^2) \bigg\} \tag{Optimality of $h_*'$}\\
& = \frac{1}{\alpha}\{\|\Tcal(h - h_*')\|^2 + \alpha \|h_*'\|^2\} \tag{First order condition of $h_*'$}\\
&\leq \frac{2}{\alpha}\bigg\{    2\|\Tcal(h-h_0)\|^2 + 2\|\Tcal(h_*'-h_0)\|^2 + 2\alpha\|h-h_0\|^2 + 2\alpha \|h_*' - h_0\|^2\bigg\}\\
&\leq \frac{2}{\alpha}\big\{4\|h-h_0\|^2 + O\big(\|w_0\|^2\alpha^{\min\{\beta+1,2\}}\big)\big\},
\end{align*}
set $h = h^{\dagger}$,  by strong convexity and $\partial\Lcal_0(0) = 0$, we have \begin{align*}
   \|h_* - h_*'\|^2 \leq \frac{1}{\alpha}|\Lcal_0(1) - \Lcal(0)|\leq O(\frac{1}{\alpha}\{\epsilon_\Hcal^2 + \alpha^{\min\{\beta+1,2\}}\}).
\end{align*}
% We further have \begin{align*}
% \|h_* - h_*'\|^2 &\leq 2\bigg(1+ \frac{1}{\alpha}\bigg) \bigg\{ \|h_*' - h_0\|^2 + \min_{h^\dagger} \|h_0 - h^{\dagger}\|^2 \bigg\}\\
% &\leq \bigg(1+ \frac{1}{\alpha}\bigg) \big\{\alpha^{\min\{\beta,2\}} + \epsilon_\Hcal^2\big\}.
% \end{align*}
% Here the second inequality comes from Lemma \ref{lem:iter-tikh-bias}, by which we have $\|h_{*}' - h_0\|_2^2 \leq \alpha^{\min\{\beta,2\}}$. 
Therefore we have \begin{align*}
    \|\hat{h} - h_0\|^2 &\leq 3\bigg\{\|\hat{h} - h_{*}\|^2 + \|h_{*} - h'_{*}\|^2 + \|h_{*}' - h_0\|^2  \bigg\}\\
    &= 3\|\hat{h} - h_{*}\|^2+O\bigg(\frac{1}{\alpha}\bigg\{\epsilon_\Hcal^2 + \alpha^{\min\{\beta+1,2\}}\bigg\}\bigg) + 3\alpha^{\min\{\beta,2\}},
\end{align*}
and we conclude our proof for the lemma.
\end{proof}
\paragraph{Proof for Theorem \ref{thm: mle-misspec}.}
By Lemma \ref{lem: quadra-lemma}, we have \begin{align*}
    \alpha \|\hat{h} - h_{*}\|^2 =  O\bigg(\delta_{n,\Hcal}\bigg\{ (\alpha+1)  \|\hat h-h_*\|_2   +  \delta_{n,\Hcal}   \bigg\}+  \|(\hat \Tcal-\Tcal)(\hat h-h_*)\|_1\bigg),
\end{align*}

By Corollary \ref{cor: T-mle}, we have $\|(\Tcal - \hat{\Tcal})(\hat{h} - h_*)\|_1  \leq (\delta_{n,\Gcal}^2 +\epsilon_\Gcal)^{1/2}\|\hat{h} - h_*\|$, 
and we have $$
\|\hat{h} - h_*\|^2 \leq \frac{1}{\alpha}\cdot O\bigg(\delta_{n,\Hcal} \|\hat{h} - h_*\| +  (\delta_{n,\Gcal}^2 +\epsilon_\Gcal )^{1/2}\|\hat{h} - h_*\| + \delta_{n,\Gcal}^2 \bigg), 
$$
therefore by Lemma \ref{lem:aux-1}, we have
\begin{align}
    \|\hat{h} - h_*\|^2 = O\bigg(\frac{\delta_{n,\Gcal}^2 + \epsilon_\Gcal + \delta_{n,\Hcal}^2}{\alpha^2}   \bigg). \label{eqn: misspec_G_only}
\end{align}
By Lemma \ref{lem:misspec-main}, combine everything together: $$
\|\hat{h} - h_0\|^2 = O\bigg(\frac{\delta_{n,\Gcal}^2+ \delta_{n,\Hcal}^2 + \epsilon_\Gcal}{\alpha^2} + \alpha^{\min\{\beta+1,2\}-1}+\frac{\epsilon_\Hcal^2}{\alpha}\bigg).
$$
note that $\delta_n :=\{\delta_{n,\Gcal}, \delta_{n,\Hcal}\}$, we conclude the proof of Theorem \ref{thm: mle-misspec}.
% \begin{remark}
% By choosing $\alpha = (\delta_n^2 + \epsilon_\Gcal)^{\frac{1}{\min\{\beta,2\}+1}}$,  we have $$
% \|\hat{h} - h_0\|^2 \leq (\delta_n^2 + \epsilon_\Gcal)^{\frac{\min\{\beta,2\}-1}{\min\{\beta,2\}+1}}+(\delta_n^2 + \epsilon_\Gcal)^{\frac{\min\{\beta,2\}}{\min\{\beta,2\}+1}} + \epsilon_{\Hcal}^2 \big\{1+ (\delta_n^2 + \epsilon_\Gcal)^{-\frac{1}{\min\{\beta,2\}+1}}\big\}.
% $$
% \end{remark}
\paragraph{Proof of Theorem \ref{thm: chi2mle-misspec}.}
By Lemma \ref{lem: quadra-lemma}, we have \begin{align*}
    \alpha \|\hat{h} - h_{*}\|^2 &\leq  O\bigg(\delta_n\bigg\{ (\alpha+1) \|\hat h-h_*\|_2  +     \delta_n   \bigg\}+  \|(\hat \Tcal-\Tcal)(\hat h-h_*)\|_1\bigg)\\
    &\leq O\bigg(\delta_n\bigg\{ (\alpha+1) \|\hat h-h_*\|_2  +     \delta_n   \bigg\}+  \|(\hat \Tcal-\Tcal)(\hat h-h_*)\|_2\bigg),
\end{align*}
by Lemma \ref{cor: T-chi2_mle}, we have $\|(\hat \Tcal-\Tcal)(\hat h-h_*)\|_2\leq (\delta_n^2+\epsilon_\Gcal)^{1/2}\|\hat h-h_*\|_\infty$, therefore we have \begin{align*}
\|\hat{h} - h_*\|^2 &\leq \frac{1}{\alpha}\cdot O\bigg(\delta_n \|\hat{h} - h_*\| +  (\delta_n^2 +\epsilon_\Gcal )^{1/2}\|\hat{h} - h_*\|_\infty\bigg)\tag{$\alpha\leq 1$}\\
&\leq \frac{1}{\alpha}\cdot O\bigg(\delta_n\|\hat{h} - h_*\|+ (\delta_n^2 +\epsilon_\Gcal )^{1/2}\|\hat{h} - h_*\|_2^\gamma\bigg),
\end{align*}
where the second inequality comes from Assumption \ref{ass:gamma-smooth}. By Lemma \ref{lem:aux-1}, we have 
\begin{align}
    \|\hat{h} - h_*\|^2 \leq O\bigg(\bigg(\frac{\delta_n^2+\epsilon_\Gcal}{\alpha^2} \bigg)^{1/(2-\gamma)}\bigg)
    \label{eqn: misspec_only_G_chi2}
\end{align}

by Lemma \ref{lem:misspec-main}, combine everything together,  we have $$
\|\hat{h} - h_0\|^2 \leq O\bigg(\bigg(\frac{\delta_n^2+\epsilon_\Gcal}{\alpha^2} \bigg)^{1/(2-\gamma)}+ \alpha^{\min\{\beta+1,2\}-1}+\frac{\epsilon_\Hcal^2}{\alpha}\bigg),
$$ 
and thus we conclude the proof of Theorem \ref{thm: chi2mle-misspec}.

% \begin{remark}
%     By setting $\alpha= (\delta_n^2 + \epsilon_\Gcal)^{\frac{1}{\min\{\beta-1,1\}(2-\gamma)+2}} $, we have $$
%     \|\hat{h} - h_0\|^2 \leq O\big((\delta_n^2 + \epsilon_\Gcal)^{\frac{\min\{\beta-1,1\}}{\min\{\beta-1,1\}(2-\gamma)+2}}+(\delta_n^2 + \epsilon_\Gcal)^{\frac{\min\{\beta,2\}}{\min\{\beta-1,1\}(2-\gamma)+2}}+\epsilon_\Hcal^2\big\{1+ (\delta_n^2 + \epsilon_\Gcal)^{-\frac{1}{\min\{\beta-1,1\}(2-\gamma)+2}}\big\}\big).
%     $$
% \end{remark}
\section{Proof of Theorem \ref{thm:model_selection}}\label{sec: proof-model-selection}
In this section, we will provide the details for the model selection results in the paper. 
Let $\ell_{h,g}(Y,Z,X)$ denote the loss evaluated for a function $h$ using the likelihood function $\hat{g}$: 
\begin{align*}
    \ell_{h,\hat{g}}(Y,Z,X) ={\left(Y - \int h(x)\hat{g}(x|Z)\mu(dx)\right)^2} + {\alpha h(X)^2}
\end{align*}
Also, to simplify the notation, we use $\{X_i,Y_i,Z_i\}$ instead of  $\{X'_i,Y'_i,Z'_i\}$. 

For $\theta\in\Theta = \{\theta | \sum_j \theta_j = 1, \theta_j\geq 0 \forall j\} $, denote $h_{\theta} = \sum_j \theta_j f_j$. For any convex combination $\theta$ over a set of candidate functions $\{h_1, \ldots, h_M\}$, we define the notation:
\begin{align*}
    \ell_{\theta,g}(Y,Z,X) :=~& \ell_{h_{\theta}, g}(Y,Z,X) &
    R(\theta,g) :=~& P\ell_{\theta,g}(Y,Z,X)
\end{align*}
Here we define some optimal aggregates in the following sense:
\begin{align*}
    j^*_{\alpha} :=~& \argmin_{j = 1, \dots, M}  R(h_{j},g_0) &
    j^* :=~& \argmin_{j = 1, \dots, M} \|h_0 - h_j\|^2\\
    \theta^*_{\alpha} :=~& \argmin_{\theta \in \Theta} R(h_{\theta},g_0) &
    \theta^* :=~& \argmin_{\theta \in \Theta} \|h_0 - h_{\theta}\|^2\\
    h_{\alpha}^* :=~& \argmin R(h,g_0) &
    h_{\alpha, \Hcal}^* :=~& \argmin_{h \in \Hcal} R(h,g_0)
\end{align*}

\begin{proof}[Proof of Theorem \ref{thm:model_selection}]
\begin{align*}
    \|h_{\hat{\theta}} - h_0\|^2 \leq ~& 2\|h_{\hat{\theta}} - h_{\alpha}^*\|^2 + 2\|h_{\alpha}^*-h_0\|^2 \tag{By Strong Convexity}\\
    \leq~& \frac{2}{\alpha}\left(R(h_{\hat{\theta}},g_0) - R(h_{\alpha}^*,g_0)\right) + O\left(\alpha^{\min\{2,\beta\}}\right)\\
    =~& \frac{2}{\alpha}\left(R(h_{\hat{\theta}},g_0) -R(h_{\alpha,\Hcal_j}^*,g_0) + R(h_{\alpha, \Hcal_j}^*,g_0)-R(h_{\alpha}^*,g_0)\right) + O\left(\alpha^{\min\{2,\beta\}}\right)\tag{for any $j$}\\
    =~& \frac{2}{\alpha} \left(R(h_{\hat{\theta}},g_0) - R(h_{j_{\alpha}^*},g_0) + R(h_{j_{\alpha}^*},g_0) -R(h_{\alpha,\Hcal_{j}}^*,g_0) + R(h_{\alpha, \Hcal_{j}}^*,g_0)-R(h_{\alpha}^*,g_0)\right) \\
    ~&+O\left(\alpha^{\min\{2,\beta\}}\right)\\
    \leq ~&\frac{2}{\alpha} \left(R(h_{\hat{\theta}},g_0) - R(h_{j_{\alpha}^*},g_0) + R(h_{j},g_0) -R(h_{\alpha,\Hcal_{j}}^*,g_0) + R(h_{\alpha, \Hcal_{j}}^*,g_0)-R(h_{\alpha}^*,g_0)\right)\\
    ~&+O\left(\alpha^{\min\{2,\beta\}}\right)\\
    % \leq ~&\frac{2}{\alpha} \left( R(h_{j},g_0) -R(h_{\alpha,\Hcal_{j}}^*,g_0) + R(h_{\alpha, \Hcal_{j}}^*,g_0)-R(h_{\alpha}^*,g_0)\right)\\
    % ~&+O\left(\alpha^{\min\{2,\beta\}} + \frac{M}{\alpha n} + \frac{\delta_{n,\Gcal}^2}{\alpha^2}\right) \tag{By Theorem \ref{thm:oracle_model_selection}}\\
\end{align*}
When $\hat{g}$ is estimated using the standard MLE appraoch, we have that by Corollary \ref{cor: T-mle} and Lemma \ref{lem: quadra-lemma}, we have that:
\begin{align*}
    R(h_{j},g_0) -R(h_{\alpha,\Hcal_{j}}^*,g_0) \leq~& c_1 \delta_{n,j}^2 + c_2 (\delta_{n,j}^2 + \epsilon_{\Gcal})^{\frac{1}{2}}\|h_j-h_{\alpha,\Hcal_{j}}^*\|\\
    \leq~& c_1 \delta_{n,j}^2 + \frac{c_2^2 (\delta_{n,j}^2 + \epsilon_{\Gcal})}{\alpha} + \frac{1}{2}\alpha \|h_j-h_{\alpha,\Hcal_{j}}^*\|^2\\
    \leq ~& O\bigg( \delta_{n,j}^2 + \frac{(\delta_{n,j}^2 + \epsilon_{\Gcal})}{\alpha}\bigg) \tag{By Eqn \ref{eqn: misspec_G_only}}
\end{align*}
Thus, we have $R(h_{j},g_0) -R(h_{\alpha,\Hcal_{j}}^*,g_0) \leq O\left(\frac{\delta_{n,j}^2 + \epsilon_{\Gcal}}{\alpha}\right)$. Instantiating this result for the function class $\Hcal_M$, which denotes the convex hull when convex-ERM is used, or the set of candidate functions when best-ERM is used, we get that:
\begin{align*}
    R(h_{\hat{\theta}},g_0) - R(h_{j_{\alpha}^*},g_0) \leq ~& R(h_{\hat{\theta}},g_0) - R(h_{\theta_{\alpha}^*},g_0)\\
    \leq ~& \frac{\delta_{n,M}^2 + \epsilon_\Gcal}{\alpha}
\end{align*} where $\delta_{n,M} = \max\{\delta_{n,\Gcal}, \delta_{n,\Hcal_M}\}$. Since the function classes used to train the candidate functions are typically more complex than the convex hull over $M$ variables, it is safe to assume that $\delta_{n,\Hcal_M}\leq\delta_{n,\Hcal}$.
Combining, we get:
\begin{align*}
    \|h_{\hat{\theta}} - h_0\|^2
    \leq ~& O\big(\alpha^{\min\{2,\beta\}} + \frac{\delta_{n,j}^2 + \epsilon_{\Gcal}}{\alpha^2}\big) 
    +\frac{2}{\alpha}\left(R(h_{\alpha, \Hcal_{j}}^*,g_0)-R(h_{\alpha}^*,g_0)\right)\\
    \leq ~& O\big(\alpha^{\min\{2,\beta\}} + \frac{\delta_{n,j}^2 + \epsilon_{\Gcal}}{\alpha^2}\big) + \frac{2}{\alpha}\left(R(h,g_0)-R(h_{\alpha}^*,g_0)\right) \tag{for any $h\in \Hcal_{j}$}\\ 
\end{align*} 
% where in the last inequality we make the safe assumption that $\delta_n$ is greater than $\frac{M}{n}$, which is reasonable as the $\frac{M}{n}$ is the critical radius for the convex combination of $M$ variables, and typically the function classes $\Hcal_j$ that are used have greater complexity, and thus larger critical radius.  

For any function class $\Hcal$, we have:
\begin{align*}
    R(h,g_0)-R(h_{\alpha}^*,g_0) = ~& \|\Tcal(h-h_{\alpha}^*)\|^2 + \alpha \|h-h_{\alpha}^*\|^2\\
    \leq ~& 2\|\Tcal(h-h_0)\|^2 + 2\|\Tcal(h_{\alpha}^*-h_0)\|^2 + 2\alpha\|h-h_0\|^2 + 2\alpha \|h_{\alpha}^* - h_0\|^2\\
    \leq~& 4\|h-h_0\|^2 + O\big(\|w_0\|^2\alpha^{\min\{\beta+1,2\}}\big) \tag{By Lemma 3 in \cite{bennett2023source}}\\
\end{align*}
Hence, for any function class $\Hcal_j$, we can choose $h$ that attains $\min_{\Hcal_j}\|h-h_0\| = \epsilon_{\Hcal_j}$. Combining, we get that: 
\begin{align*}
    \|h_{\hat{\theta}} - h_0\|^2 \leq \min_{j}O\big(\alpha^{\min\{\beta+1,2\}-1} + \frac{\delta_{n,j}^2 + \epsilon_\Gcal}{\alpha^2} + \frac{1}{\alpha} \epsilon_{\Hcal_j}^2\big). \tag{$\alpha \leq 1$}
\end{align*}

Analogously, if $\hat{g}$ is estimated using $\chi^2$-MLE, we have that by Corollary \ref{cor: T-chi2_mle}, Lemma \ref{lem: quadra-lemma} and Assumption \ref{ass:gamma-smooth}:
\begin{align*}
    R(h_{j},g_0) -R(h_{\alpha,\Hcal_{j}}^*,g_0) \leq~& O\bigg(\delta_{n}^2+ (\delta_{n}^2 +\epsilon_\Gcal )^{1/2}\|\hat{h} - h_*\|_2^\gamma\bigg)\\
    \leq~& O\bigg(\delta_{n}^2+ \left(\frac{\delta_{n}^2 +\epsilon_\Gcal}{\alpha^{\gamma}}\right)^{\frac{1}{1-2\gamma}} + \alpha\|\hat{h} - h_*\|_2^2\bigg) \tag{By Young's Inequality}\\
    \leq~& O\bigg(\alpha \left(\frac{\delta_{n}^2 +\epsilon_\Gcal}{\alpha^{1-\gamma}}\right)^{\frac{1}{1-2\gamma}} + \alpha\bigg(\frac{\delta_{n}^2+\epsilon_\Gcal}{\alpha^2} \bigg)^{1/(2-\gamma)}\bigg) \tag{By Eqn \ref{eqn: misspec_only_G_chi2}}\\
    \leq~& O\bigg(\alpha\bigg(\frac{\delta_{n}^2+\epsilon_\Gcal}{\alpha^2} \bigg)^{1/(2-\gamma)}\bigg)
\end{align*}
By the same argument for the standard MLE case, we get:
\begin{align*}
    \|h_{\hat{\theta}} - h_0\|^2 \leq \min_{j}O\left(\alpha^{\min\{\beta+1,2\}-1} + \left(\frac{\delta_{n,M}^2+\epsilon_\Gcal}{\alpha^2} \right)^{1/(2-\gamma)} + \frac{1}{\alpha} \epsilon_{\Hcal_j}^2\right)
\end{align*}

\end{proof}

\section{Proof of Theorem \ref{thm: mle-nonmisspec} and \ref{thm:chi2-mle-result}}\label{sec: proof-iterative}
In this section, we prove the convergence rate of iterative RMIV in Section \ref{sec: iterative} under a unified framework. We prove the results of Theorem \ref{thm:chi2-mle-result} and \ref{thm:chi2-mle-result} respectively. Recall that we define \begin{align*}
    h_{m, *} = &\argmin_{h\in\Hcal} \E[{Y - \Tcal h(Z)}^2]+ \alpha\cdot \E[(h - h_{m-1,*})^2(X)],
\end{align*}
by Lemma \ref{lem:iter-tikh-bias} and Assumption \ref{ass:source-cond}, we have $$
\|h_{m,*} - h_0\|_2^2 \leq \|w_0\|_2^2\alpha^{\min\{\beta,2m\}}.
$$
Therefore, we only need to provide a upper bound for $\|\hat{h}_m - h_{m,*}\|_2^2$, and then choose the proper $\alpha$ deliberately. We start by proving the following lemma, and with the different convergence rate of MLE and $\chi^2$-MLE, we conclude the proof of Theorem \ref{thm: mle-nonmisspec} and Theorem \ref{thm:chi2-mle-result} respectively.
\begin{lemma}\label{lem: iter-quadra}
   We have the following inequality holds with probability at least $1 - m\exp(n\delta_{n,\Hcal}^2)$: $$\|\hat{h}_{m} - h_{m,*}\|^2 \leq O\big(\delta_{n,\Hcal}^2/\alpha^2 
\big) + O\bigg(\frac{ \mathbb{E}[|(\Tcal - \hat{\Tcal})(\hat{h}_{m} - h_{m,*})|]}{\alpha}\bigg)+ 16\|\hat{h}_{m-1} - h_{m-1,*}\|^2.
$$
\end{lemma}
\begin{proof}
    Recall that our solution $\hat{h}_{m}$ satisfies $$
\hat{h}_{m} = \argmin_{h\in\Hcal} L(\hat{\Tcal}h) + \alpha \mathbb{E}_n [\{h - \hat{h}_{m-1}\}^2].
$$
We define \begin{align*}
L_m(\tau) &= \mathbb{E}[\mathbb{E}[h_0 - h_{m,*} - \tau (\hat{h}_{m} - h_{m,*})\mid Z]^2]+\alpha \|h_{m,*} + \tau (\hat{h}_{m} - h_{m,*}) - h_{m-1,*}\|^2,
\end{align*}
By definition, $L_m(\tau)$ is minimized by $\tau = 0$. Note that by strong convexity and property of quadratic function, we have $$
L_m(1) - L_m(0) = L'(0) + L''(0) \geq L''(0),
$$
Therefore \begin{align*}
    &\alpha \|\hat{h}_{m} - h_{m,*}\|^2 + \|\Tcal(\hat{h}_{m} - h_{m,*})\|^2\\
    & \leq\|\Tcal(h_0 - \hat{h}_{m})\|^2 - \|\Tcal(h_0 - h_{m,*})\|^2 + \alpha\big(\|\hat{h}_{m} - h_{m-1,*}\|^2 - \|h_{m,*} - h_{m-1,*}\|^2\big)\\
    & =\mathbb{E}[L(\Tcal\hat{h}_{m})] - \mathbb{E}[L(\Tcal h_{m,*})]+ \alpha\big(\|\hat{h}_{m} - h_{m-1,*}\|^2 - \|h_{m,*} - h_{m-1,*}\|^2\big),\\
    % &\qquad+ \alpha\big(\|\hat{h}_{m} - h_{m-1,*}\|^2 - \|h_{m,*} - h_{m-1,*}\|^2\big).
    % &\leq\mathbb{E}[-2Y\cdot\Tcal(\hat{h}_{m} - h_{m,*}) + (\Tcal\hat{h}_{m})^2 - (\Tcal h_{m,*})^2] + \alpha\big(\|\hat{h}_{m} - h_{m-1,*}\|^2 - \|h_{m,*} - h_{m-1,*}\|^2\big)\\
    % &\leq \mathbb{E}[-2Y\cdot\hat{\Tcal}(\hat{h}_{m} - h_{m,*}) + (\hat{\Tcal}\hat{h}_{m})^2 - (\hat{\Tcal} h_{m,*})^2] + \alpha\big(\|\hat{h}_{m} - h_{m-1,*}\|^2 - \|h_{m,*} - h_{m-1,*}\|^2\big) \\
    % &\qquad + c\cdot\mathbb{E}[|(\Tcal - \hat{\Tcal})(\hat{h}_{m} - h_{m,*})|]+ \alpha\big(\|\hat{h}_{m} - h_{m-1,*}\|^2 - \|h_{m,*} - h_{m-1,*}\|^2\big)\\
    % & = \mathbb{E}[L(\hat{\Tcal}\hat{h}_{m})] - \mathbb{E}[L(\hat{\Tcal} h_{m,*})] + c\cdot\mathbb{E}[|(\Tcal - \hat{\Tcal})(\hat{h}_{m} - h_{m,*})|]\\
    % &\qquad+ \alpha\big(\|\hat{h}_{m} - h_{m-1,*}\|^2 - \|h_{m,*} - h_{m-1,*}\|^2\big).
\end{align*}
and thus we have \begin{align*}
&\alpha \|\hat{h}_{m} - h_{m,*}\|^2 + \|\Tcal(\hat{h}_{m} - h_{m,*})\|^2\\
    & \leq \mathbb{E}[L(\hat{\Tcal}\hat{h}_{m})] - \mathbb{E}[L(\hat{\Tcal} h_{m,*})] + c\cdot\mathbb{E}[|(\Tcal - \hat{\Tcal})(\hat{h}_{m} - h_{m,*})|]\\
    &\qquad+ \alpha\big(\|\hat{h}_{m} - h_{m-1,*}\|^2 - \|h_{m,*} - h_{m-1,*}\|^2\big)\\
    & \leq |(\mathbb{E} - \mathbb{E}_n)(L(\hat{\Tcal}\hat{h}_{m}) - L(\hat{\Tcal} h_{m,*}) )| + \EE_n[L(\hat{\Tcal}\hat{h}_{m}) - L(\hat{\Tcal} h_{m,*})] \\
    &\qquad+ c\cdot\mathbb{E}[|(\Tcal - \hat{\Tcal})(\hat{h}_{m} - h_{m,*})|] +  \alpha\big(\|\hat{h}_{m} - h_{m-1,*}\|^2 - \|h_{m,*} - h_{m-1,*}\|^2\big)\\
    &\leq c_1 (\delta_n\|\hat{\Tcal}(\hat{h}_{m} - h_{m,*})\| + \delta_n^2)+ \EE_n[L(\hat{\Tcal}\hat{h}_{m}) - L(\hat{\Tcal} h_{m,*})] + c \cdot \mathbb{E}[|(\Tcal - \hat{\Tcal})(\hat{h}_{m} - h_{m,*})|] \\
    &\qquad +  \alpha\big(\|\hat{h}_{m} - h_{m-1,*}\|^2 - \|h_{m,*} - h_{m-1,*}\|^2\big),\\
\end{align*}
holds for all $m$ simultaneously with probability at least $1- m \exp(n\delta_{n,\Hcal}^2)$, recall that $\delta_{n,\Hcal}^2$ is the critical radius.  Here the second inequality comes from triangular inequality and $L(\cdot)$ being $O(1)$-Lipschitz, the third inequality comes from  Lemma \ref{lem:local-concen}.
By Eq.~\eqref{eq: emp-iter},  $$
 \EE_n[L(\hat{\Tcal}\hat{h}_{m}) - L(\hat{\Tcal} h_{m,*})] \leq  \alpha(\|h_{m,*} - \hat{h}_{m-1}\|^2_n - \|\hat{h}_{m} - \hat{h}_{m-1}\|^2_n),
$$
therefore we have \begin{align*}
&\alpha \|\hat{h}_{m} - h_{m,*}\|^2 + \|\Tcal(\hat{h}_{m} - h_{m,*})\|^2\\
&\leq c_1 (\delta_n\|(\hat{h}_{m} - h_{m,*})\| + \delta_n^2) + \alpha(\|h_{m,*} - \hat{h}_{m-1}\|^2_n - \|\hat{h}_{m} - \hat{h}_{m-1}\|^2_n) + c \cdot \mathbb{E}[|(\Tcal - \hat{\Tcal})(\hat{h}_{m} - h_{m,*})|] \\
&\qquad +  \alpha\big(\|\hat{h}_{m} - h_{m-1,*}\|^2 - \|h_{m,*} - h_{m-1,*}\|^2\big).
\end{align*}
We are now interested in bounding 
$$
(\|h_{m,*} - \hat{h}_{m-1}\|^2_n - \|\hat{h}_{m} - \hat{h}_{m-1}\|^2_n) +\big(\|\hat{h}_{m} - h_{m-1,*}\|^2 - \|h_{m,*} - h_{m-1,*}\|^2\big).
$$
We divide it into two terms:\begin{align*}
    &I_1 := \big(\|h_{m,*} - \hat{h}_{m-1}\|^2 - \|\hat{h}_{m} - \hat{h}_{m-1}\|^2\big) + \big(\|\hat{h}_{m} - h_{m-1,*}\|^2 - \|h_{m,*} - h_{m-1,*}\|^2\big),\\
    &I_2 := (\|h_{m,*} - \hat{h}_{m-1}\|^2_n - \|\hat{h}_{m} - \hat{h}_{m-1}\|^2_n) - (\|h_{m,*} - \hat{h}_{m-1}\|^2 - \|\hat{h}_{m} - \hat{h}_{m-1}\|^2)
\end{align*}
Note that $|I_1| = \big|2\langle \hat{h}_{m-1}- h_{m-1,*}, \hat{h}_{m}- h_{m,*} \rangle\big| $, we have $$
I_1 \leq 2\|\hat{h}_{m-1}- h_{m-1,*}\|_2 \|\hat{h}_{m}- h_{m,*}\|_2,
$$
For $I_2$, we divide it into two terms $I_3$ and $I_4$, defined by \begin{align*}
    &I_3 := \|h_{m,*} - \hat{h}_{m-1}\|_{n}^2 - \|h_{m,*} - h_{m-1,*}\|_n^2 - (\|h_{m,*} - \hat{h}_{m-1}\|^2 - \|h_{m,*} - h_{m-1,*}\|^2),\\
    &I_4 := \|h_{m,*} - h_{m-1,*}\|_n^2 - \|\hat{h}_{m} - \hat{h}_{m-1}\|_n^2 -(\|h_{m,*} - h_{m-1,*}\|^2 - \|\hat{h}_{m} - \hat{h}_{m-1}\|^2)
\end{align*}
Since each of these is the difference of two centered empirical processes, that are also Lipschitz losses (since
$h_{m,*}, \hat{h}_m, h_{m-1,*}, \hat{h}_{m-1}$ are uniformly bounded) and since $h_{m,*}$ is a population quantity and not dependent on the empirical sample that is used for the $m$-th iterate, we can also upper bound these,
\begin{align*}
    &I_3 = O(\delta_{n,\Hcal}^2 \|\hat{h}_{m-1} - h_{m-1,*}\|+ \delta_{n,\Hcal}^2),\\
    &I_4 = O(\delta_{n,\Hcal}\|\hat{h} - h_{m,*} + h_{m-1,*} - \hat{h}_{m-1}\| +\delta_{n,\Hcal}^2) = O\bigg(\delta_{n,\Hcal}(\|\hat{h} - h_{m,*}\|+ \|h_{m-1,*} - \hat{h}_{m-1}\|+ \delta_{n,\Hcal}^2)\bigg),
\end{align*}
combine everything together,  we can prove that  \begin{align*}
&(\|h_{m,*} - \hat{h}_{m-1}\|^2_n - \|\hat{h}_{m} - \hat{h}_{m-1}\|^2_n) +\big(\|\hat{h}_{m} - h_{m-1,*}\|^2 - \|h_{m,*} - h_{m-1,*}\|^2\big)\\
&\qquad\qquad\leq O(\delta_n^2 + \delta_n(\|\hat{h}_{m} - h_{m,*}\| + \|\hat{h}_{m-1} - h_{m-1,*}\|)) + 2\|\hat{h}_{m-1} - h_{m-1,*}\|\|\hat{h}_{m} - h_{m,*}\|.
\end{align*}
Therefore, we have \begin{align*}
    &\alpha \|\hat{h}_{m} - h_{m,*}\|^2 \\
    &\qquad\leq O\bigg(\delta_{n,\Hcal}^2 + \delta_{n,\Hcal}\|\hat{h}_{m} - h_{m,*}\| + c \cdot \mathbb{E}[|(\Tcal - \hat{\Tcal})(\hat{h}_{m} - h_{m,*})|] + \alpha\delta_{n,\Hcal}(\|\hat{h}_{m} - h_{m,*}\| + \|\hat{h}_{m-1} - h_{m-1,*}\|))\bigg)\\
    &\qquad\qquad + 2\alpha\|\hat{h}_{m-1} - h_{m-1,*}\|\|\hat{h}_{m} - h_{m,*}\|.
\end{align*}
By applying AM-GM inequality and utilizing $\alpha\leq 1$, we have $$
\frac{\alpha}{8}\|\hat{h}_{m} - h_{m,*}\|^2 \leq O\big(\delta_n^2/\alpha + \delta_n^2 
+\alpha\delta_n^2\big) + c \cdot \mathbb{E}[|(\Tcal - \hat{\Tcal})(\hat{h}_{m} - h_{m,*})|] + 2\alpha\|\hat{h}_{m-1} - h_{m-1,*}\|^2,
\label{eqn:strong-iterated}$$
therefore we have $$
\|\hat{h}_{m} - h_{m,*}\|^2 \leq O\big(\delta_n^2/\alpha^2 + \delta_n^2 /\alpha
\big) + O\bigg(\frac{ \mathbb{E}[|(\Tcal - \hat{\Tcal})(\hat{h}_{m} - h_{m,*})|]}{\alpha}\bigg)+ 16\|\hat{h}_{m-1} - h_{m-1,*}\|^2.
$$
\paragraph{Proof for Theorem \ref{thm: mle-nonmisspec}.}
By Corollary \ref{cor: T-mle}, we have $$
\mathbb{E}[|(\Tcal - \hat{\Tcal})(\hat{h}_{m} - h_{m,*})|] = \|(\Tcal - \hat{\Tcal})(\hat{h}_{m} - h_{m,*})\|_1 \leq \delta_{n}\cdot\|\hat{h}_{m} - h_{m,*})\|_2,
$$
therefore by Lemma \ref{lem: iter-quadra}, we have \begin{align*}
    \|\hat{h}_{m} - h_{m,*}\|^2 \leq O(\delta_n^2/\alpha^2 + \delta_n\|\hat{h}_{m} - h_{m,*})\|_2)+ 16 \|\hat{h}_{m-1} - h_{m-1,*}\|^2. 
\end{align*}
By Lemma \ref{lem:aux-1}, we have \begin{align*}
    \|\hat{h}_m - h_{m,*}\|^2 &\leq 4O\big(\delta_n^2/\alpha^2\big) + 16\|\hat{h}_{m-1} - h_{m-1,*}\|^2 \\
&\leq 128^m\cdot\delta_n^2/\alpha^2,
\end{align*}
where the second inequality comes from induction. Therefore, by Lemma \ref{lem:iter-tikh-bias}, we have $$
\|\hat{h}_m - h_{0}\|^2 =O(128^m\cdot\delta_n^2/\alpha^2 + \alpha^{\min\{\beta, 2m\}}).
$$
Set $\alpha = \delta_n^{\frac{2}{2+\min\{\beta, 2m\}}}$, and we conclude the proof.
\paragraph{Proof for Theorem \ref{thm:chi2-mle-result}}
By Assumption \ref{ass:gamma-smooth}, we have $\|\hat{h}_{m} - h_{m,*}\|_\infty \leq\|\hat{h}_{m} - h_{m,*}\|_2^\gamma $, which implies $$
\|\hat{h}_{m} - h_{m,*}\|^2 \leq O\big(\delta_n^2/\alpha^2 + \delta_n^2 /\alpha
\big) + O\bigg({\delta_n}/{\alpha} \cdot\|\hat{h}_{m} - h_{m,*}\|^\gamma\bigg)+ 16\|\hat{h}_{m-1} - h_{m-1,*}\|^2,
$$
by Lemma \ref{lem:aux-1}, we have \begin{align*}
    \|\hat{h}_{m} - h_{m,*}\|^2 &\leq 4\max\big\{O\big(\delta_n^2/\alpha^2 + 16\|\hat{h}_{m-1} - h_{m-1,*}\|^2 
\big), O\big((\delta_n/\alpha)^{2/(2-\gamma)}\big)\big\}\\
&\leq O(128^m \max\big\{\delta_n^2/\alpha^2, (\delta_n/\alpha)^{2/(2-\gamma)}\big\}),
\end{align*}
where the second inequality comes from induction.  Therefore, by Lemma \ref{lem:iter-tikh-bias}, we have $$
\|\hat{h}_m - h_{0}\|^2 =O(128^m\cdot\max\big\{\delta_n^2/\alpha^2, (\delta_n/\alpha)^{2/(2-\gamma)}\big\} + \alpha^{\min\{\beta, 2m\}}).
$$
Set $\alpha = \delta_n^{\frac{2}{2+ (2-\gamma)\min\{\beta, 2m\}}}$, Then $\delta_n / \alpha = O(\delta_n^{\frac{(2-\gamma)\min\{\beta, 2m\}}{2+ (2-\gamma)\min\{\beta, 2m\}}}) \lesssim 1$, and since $\gamma \in (0,1)$, we have $$
\max\big\{\delta_n^2/\alpha^2, (\delta_n/\alpha)^{2/(2-\gamma)}\big\} = (\delta_n/\alpha)^{2/(2-\gamma)},
$$
and $$
\|\hat{h}_m - h_{0}\|^2 = O(128^m \cdot \delta_n^{\frac{2\min\{\beta, 2m\}}{2+ (2-\gamma)\min\{\beta, 2m\}}}),
$$
and we conclude the proof of Theorem \ref{thm:chi2-mle-result}.
\end{proof}

\section{Convergence rate of MLE and $\chi^2$-MLE}
\subsection{Convergence rate of MLE}
In this section, we aim to characterize the convergence rate of conditional MLE \eqref{eq:mle-est} in terms of the critical radius $\delta_{n,\Gcal}$ of function class $\Gcal$ and model misspecification. Specifically, we prove the following Theorem: 
\begin{theorem}[Convergence rate for misspecified MLE]\label{thm:mle-conv-rate}
    Suppose Assumption \ref{ass:lower-bound} and condition in Theorem \ref{thm: mle-nonmisspec-noniter} holds, and there exists $g^{\dagger} \in \Gcal$ such that $\E_{z\sim g_0}[\kl(g_0(\cdot|z), g^{\dag}(\cdot|z))] \leq \epsilon_\Gcal$. Then we have $$
    \E_{z\sim g(z)}\big[H^2(\hat{g}(\cdot|z)| g_0(\cdot|z))\big] \leq \delta_n^2 + \epsilon_{\Gcal}
    $$
    holds with probability at least $1- c_1\exp(c_2 \frac{c_0}{C+c_0}n\delta_n^2)$.
\end{theorem}
\begin{proof}
 We work with the transformed function class $\Fcal = \bigg\{ \sqrt{\frac{g+ g_0}{2g_0}}\biggl\vert g\in \Gcal \bigg\}$, and define $\Lcal_f = -\log f(x)$ for $f\in \Fcal$. Note that $\Fcal$ is a function class whose element maps $\Xcal\times \Zcal$ to $\RR$. We define the population version of localized Rademacher complexity for function class $\Fcal^* := \operatorname{star}((\Fcal - f^*) \cup \{0\})$. By Assumption \ref{ass:lower-bound} and 1-boundedness of $\Gcal$,  $\Fcal$ and $\Fcal^*$ are bounded by a constant $b:= \frac{C_0+C}{2C_0}$ in $\|\cdot\|_\infty$.
 The critical radius $\delta_{n,\Fcal}$ of function class $\Fcal^*$ is any solution such that \begin{align*}
    \delta^2 \geq c/n \text{ and }\bar{R}_n (\delta; \Fcal^*)\leq \delta^2/b.
\end{align*}
Such critical radius can be easily calculated for a large number of function classes. For example, we can use $$
\frac{64}{\sqrt{n}}\int_{\delta^2/2b}^{\delta}\sqrt{\log N_n(t, \Bcal(\delta,\Fcal^*))}dt \leq \frac{\delta^2}{b}
$$
to calculate $\delta_{n,\Fcal}$, where $\Bcal(\delta,\Fcal^*) := \{f\in\Fcal^*\mid \|f\|_2\leq \delta\}$, $N_n$ is the empirical covering number conditioned on $\{(x_i,z_i)\}_{i\in[n]}$.
For a cost function $\Lcal : \mathbb{R} \rightarrow \mathbb{R}$, we define $\Lcal_f(x,z):= \Lcal(f(x,z))$.
We make the following definition. \begin{definition}\label{def:strong-convexity}
        We say $\Lcal_f$ is $\gamma$-strongly convexity at $f^*$ if $$
\E_{z\sim g_0(z), x\sim g_0(x\mid z)}\big[\Lcal_f(x,z) - \Lcal_{f^*}(x,z) - \partial\Lcal_{f^*}(x,z) (f - f^*)(x,z) \big] \geq \frac{\gamma}{2} \|f - f^*\|^2_2
$$ for all $f\in\Fcal$.
\end{definition} 
Note that for any $f\in\Fcal$ we have 
and $|\log f (x) - \log f'(x)|\leq \sqrt{2}|f(x) - f'(x)|$ since $\|f\|_\infty \geq 1/\sqrt{2}$. By the definition of Hellinger distance, we have  $$
\| f - f^*\|_2^2 = \E_{z\sim g_0(z)}\bigg[H^2\bigg(\frac{g+ g_0}{2}| g_0\bigg)\bigg],
$$
and since $H^2(g_1\mid g_2) \leq 2 \kl(f_1\mid f_2)$, we have $\|f- f^*\|_2^2 \leq \PP(\Lcal_f - \Lcal_{f^*})$, thus $\Lcal$ is $2$-strongly convex at $f^*$.
Utilizing strong convexity and Lemma \ref{lem:local-concen}, we have the following inequality holds with probability $1- \exp(n\delta_{n,\Fcal}^2)$: \begin{align*}
    \|\hat{f} - f_0\|_2^2 &\leq 2\E_{z\sim g_0(z), x\sim g_0(x|z)}[\Lcal_{\hat{f}}(x,z) - \Lcal_{f_0}(x,z)]\\
    &= 2\E_{z\sim g_0(z), x\sim g_0(x|z)}[\Lcal_{\hat{f}}(x,z) - \Lcal_{f^{\dagger}}(x,z)] + 2\E_{z\sim g_0(z), }[\kl(g_0(\cdot|z)\mid(g^{\dagger}+g_0)/2(\cdot|z) )]\\
    &\leq 2(\E_n - \E)[\Lcal_{\hat{f}}(x,z) - \Lcal_{f^\dagger}(x,z)] + 2\E_n[\Lcal_{\hat{f}}(x,z) - \Lcal_{f^\dagger}(x,z)]\\
    &\quad+ \E_{z\sim g_0(z) }[\kl(g_0(\cdot|z)\mid g^{\dagger}(\cdot|z) )] \\
    &\leq O(\delta_{n,\Fcal}\|\hat{f} - f^{\dagger}\|_2 +\delta_{n,\Fcal}^2 ) +   \E_{z\sim g_0(z) }[\kl(g_0(\cdot|z)\mid g^{\dagger}(\cdot|z) )]\\
    &\leq O(\delta_{n,\Fcal}\|\hat{f} - f_0\|_2 +\delta_{n,\Fcal}\|f_0 - f^{\dagger}\|_2 +\delta_{n,\Fcal}^2 ) +   \E_{z\sim g_0(z) }[\kl(g_0(\cdot|z)\mid g^{\dagger}(\cdot|z) )],
\end{align*}
here the first inequality comes from strong convexity, the third inequality comes from $\log(\frac{2x}{x+y})\leq \frac{1}{2}\log(\frac{x}{y})$ and the definition of MLE. The forth inequality comes from Lemma \ref{lem:local-concen}. Solve this inequality, and recall that $\|f - h_0\|_2^2 = \E_{z\sim g_0(z)}[H^2((g+g_0)(\cdot|z)/2\mid g_0(\cdot|z))]$, we have \begin{align*}
\E_{z\sim g_0(z)}[H^2(\hat{g}(\cdot|z)\mid g_0(\cdot|z))]&\leq O(\delta_{n,\Fcal}^2 + \delta_{n,\Fcal}\|f_0 - f^{\dagger}\|_2+ \E_{z\sim g_0(z)}[\kl(g_0(\cdot|z)\mid g^{\dagger}(\cdot|z) )])\\
&\leq O(\delta_{n,\Fcal}^2 + \delta_{n,\Fcal}\E_{z\sim g_0(z)}[\kl(g_0(\cdot|z), g^{\dagger}(\cdot|z) )]^{1/2}\\
&\quad+\E_{z\sim g_0(z)}[\kl(g_0(\cdot|z)\mid g^{\dagger}(\cdot|z) )])\\
&\leq O(\delta_{n,\Fcal}^2+\E_{z\sim g_0(z)}[\kl(g_0(\cdot|z)\mid g^{\dagger}(\cdot|z) )]),
\end{align*}
here the first inequality comes from Lemma \ref{lem:weighted-hellinger}, the second inequality comes from Lemma \ref{lem:h-kl}.
Thus we conclude the proof of Theorem \ref{thm:mle-conv-rate}.
\end{proof}
We provide the following corollary, which would help characterize the $L_1$ and $L_2$ error of $\Tcal h$ introduced by MLE.
\begin{corollary}\label{cor: T-mle}
Under Assumption \ref{ass:lower-bound}, for all $h' \in\Hcal - \Hcal$, we have $\|(\hat \Tcal-\Tcal)h'\|_1 \leq \{1/c_0+1\} \|h'\|_2 \cdot (\delta_{n,\Hcal}^2 + \epsilon_\Gcal)^{1/2}$ and $\|(\hat \Tcal-\Tcal)h'\|_2 \leq (C_{2,4}C)^{1/2}\cdot (C/c_0 +1)\|h'\|_{2}\cdot (\delta_{n,\Gcal}^2 + \epsilon_\Gcal)^{1/4}$
     with probability at least $1 - c_2\exp(c_3 n\delta_{n,\Gcal}^2 )$.
\end{corollary}
\begin{proof}
We first prove the bound for $L_1$ error $\|(\hat \Tcal-\Tcal)h'\|_1$. We have the following inequality:
\begin{align*}
    \|(\hat \Tcal-\Tcal)h'\|_1 &= \EE_{z \sim g_0(z)}\left[|\EE_{x \sim g_0(x|z)}\left [\frac{\hat g(x|z)}{g_0(x|z)}h'(x) - h'(x) \right] | \right]\\
    &\leq \EE_{z \sim g_0(z),x \sim g_0(x|z)}\left[ |\frac{\hat g(x|z)}{g_0(x|z)}h'(x) - h'(x)  | \right]\\
    &\leq \EE_{z \sim g_0(z),x \sim g_0(x|z)}\left[\sqrt{\frac{\hat g(x|z)}{g_0(x|z)}}|h'(x)| |\sqrt{\frac{\hat g(x|z)}{g_0(x|z)}} - 1 |\right]\\
    &\quad + \EE_{z \sim g_0(z),x \sim g_0(x|z)}\left[|h'(x)||\sqrt{\frac{\hat g(x|z)}{g_0(x|z)}} - 1  |\right] \\
    &\leq \EE[\frac{\hat g(x|z)}{g_0(x|z)}h'^2(x) ]^{1/2} \times \EE[2H^2(\hat g(\cdot|z)\mid g_0(\cdot |z))  ] \\
    &\quad +\EE[\frac{\hat g(x|z)}{g(^{\star}(x|z)}h'^2(x) ]^{1/2} \cdot \EE[2H^2(\hat g(\cdot|z)\mid g_0(\cdot |z))  ]^{1/2} \tag{CS inequality} \\
    &\leq  2\{1/c_0+1\} \EE[h^2(x) ]^{1/2} \cdot \EE[2H^2(\hat g(\cdot|z)\mid g_0(\cdot |z))  ]^{1/2}\\
    & = \{1/c_0+1\} \|h'\|_2 \cdot (\delta_{n,\Gcal}^2 + \epsilon_\Gcal)^{1/2}.
\end{align*}
where the second inequality comes from Assumption \ref{ass:lower-bound}.
Next, we prove the upper bound for $L_2$ error $\|(\hat \Tcal-\Tcal)h'\|_2$. We have 
\begin{align*}
    \|(\hat \Tcal-\Tcal)h'\|_2 &= \big\{\E[|(\Tcal - \hat{\Tcal})h'|^2]\big\}^{1/2}\\
    &\leq 2C_Y \|(\Tcal - \hat{\Tcal})h'\|_1^{1/2}\\
    &\leq 2C_Y \delta_{n,\Hcal}^{1/2} \|h'\|^{1/2}.
\end{align*}
%     \begin{align*}
%      \|(\hat \Tcal - \Tcal)h'\|_2^2
%     & = \EE_{z\sim g_0(z)}\bigg[\bigg|\EE_{x\sim g^*(x\mid z)}\bigg[\bigg(\frac{\hat{g}(x\mid z)}{g_0(x\mid z)} - 1\bigg)h'(x)\bigg]\bigg|^2\bigg]\\
%     &\leq \EE[h'^4(x)]^{1/2}\EE[\{\hat g(x|z)/g_0(x|z)-1\}^4]^{1/2}\\
%     &\leq C_{2,4}\cdot C(C/c_0 +1)^2\cdot\|h'\|_2^2\cdot \EE_{z\sim g_0(z)}[H(g(\cdot|z),g_0(\cdot|z) )^2]^{1/2}\\
%     &\leq C_{2,4}\cdot C(C/c_0 +1)^2\cdot\|h'\|_2^2\cdot (\delta_{n,\Gcal}^2 + \epsilon_\Gcal)^{1/2},
% \end{align*}
and we conclude the proof.
\end{proof}
\subsection{Convergence rate of $\chi^2$-MLE}
For the convergence rate of $\chi^2$-MLE, we present the following theorem: \begin{theorem}[Convergence rate for $\chi^2$-MLE, Corollary 14.24 of \cite{wainwright2019high} ]
For $\hat{g}$ generated by \ref{eq:chi2-mle-est}, we have 
    $$\EE_{z \sim g_0(z)}\left[\{\int |\hat g(x|z)-g_0(x|z)| \mathrm{d}\mu(x)\}^2\right]=O\bigg(\delta_{n,\Gcal}^2+\inf_{g\in \Gcal}\EE_{z \sim g_0(z)}\left[\{\int | g(x|z)-g_0(x|z)| \mathrm{d}\mu(x)\}^2\right]\bigg)$$
with probability at least $1 - c_1 \exp(c_2n\delta_{n,\Gcal}^2)$.
\end{theorem}
\begin{proof}
    By Theorem 13.13 of \citet{wainwright2019high}, we have $$
    \E_n\left[\{\int |\hat g(x|z)-g_0(x|z)| \mathrm{d}\mu(x)\}^2\right] = O\bigg(\delta_{n,\Gcal}^2+\inf \E_n\left[\{\int | g(x|z)-g_0(x|z)| \mathrm{d}\mu(x)\}^2\right]\bigg)
    $$ holds with probability at least $1- \exp(c_1 n\delta^2_{n,\Gcal})$. By Theorem \ref{thm:localization}, we have $$
    (\E_n - \E)\left[\{\int | g(x|z)-g_0(x|z)| \mathrm{d}\mu(x)\}^2\right] \leq O(\delta_{n,\Fcal}^2)
    $$
    holds for all $g\in\Gcal$ with probability at least $1 - c_2\exp(c_3 n\delta_{n,\Gcal}^2 )$, and the proof is done.
    and the proof is done.
\end{proof}
We provide the following corollary, which would help characterize the error introduced by $\chi^2$-MLE.
\begin{corollary}\label{cor: T-chi2_mle}
    With $\chi^2$-MLE, we have the following inequality holds for all $h\in\Hcal$ with probability at least $1 - c_2\exp(c_3 n\delta_{n,\Gcal}^2 )$: \begin{align*}
    \|(\Tcal - \hat{\Tcal}) h\|_2^2 \leq (\delta_{n,\Gcal}^2+\epsilon_\Gcal) \|h\|_\infty^2 .
\end{align*}
\end{corollary}
\begin{proof}
By
    \begin{align*}
    \|(\Tcal - \hat{\Tcal}) h\|_2^2 &= \E_{z\sim g_0(z)}\bigg[\bigg(\int_{\Xcal} \{\hat{g}(x|z) - g_0(x| z)\}h(x)d\mu(x)\bigg)^2\bigg]\\
    &\leq (\delta_{n,\Gcal}^2+\epsilon_\Gcal) \|h\|_\infty^2 .
\end{align*}
We conclude the proof.
\end{proof}
\section{Auxiliary Lemma}
We introduce the following lemma, which gives a uniform convergence rate of loss error.
\begin{lemma}[Localized Concentration, \cite{foster2019orthogonal}]\label{lem:local-concen}
For any $f \in \mathcal{F}:=\times_{i=1}^d \mathcal{F}_i$ be a multivalued outcome function, that is almost surely absolutely bounded by a constant. Let $\ell(Z ; f(X)) \in \mathbb{R}$ be a loss function that is $O(1)$-Lipschitz in $f(X)$, with respect to the $\ell_2$ norm. Let $\delta_n=\Omega\left(\sqrt{\frac{d \log \log (n)+\log (1 / \zeta)}{n}}\right)$ be an upper bound on the critical radius of $\operatorname{star}\left(\mathcal{F}_i\right)$ for $i \in[d]$. Then for any fixed $h_0 \in \mathcal{F}$, w.p. $1-\zeta$ :
$$
\forall f \in \mathcal{F}:\left|\left(\mathbb{E}_n-\mathbb{E}\right)\left[\ell(Z ; f(X))-\ell\left(Z ; h_0(X)\right)\right]\right|=O\left(d \delta_n \sum_{i=1}^d\left\|f_i-f_{i, 0}\right\|_{2}+d \delta_n^2\right)
$$
If the loss is linear in $f(X)$, i.e. $\ell\left(Z ; f(X)+f^{\prime}(X)\right)=\ell(Z ; f(X))+\ell\left(Z ; f^{\prime}(X)\right)$ and $\ell(Z ; \alpha f(X))=$ $\alpha \ell(Z ; f(X))$ for any scalar $\alpha$, then it suffices that we take $\delta_n=\Omega\left(\sqrt{\frac{\log (1 / \zeta)}{n}}\right)$ that upper bounds the critical radius of $\operatorname{star}\left(\mathcal{F}_i\right)$ for $i \in[d]$.
\end{lemma}
\begin{proof}
    For a detailed proof, please refer to \cite{foster2019orthogonal}.
\end{proof}
The following lemma is useful when proving the convergence rate of Hellinger distance.
\begin{lemma}[Lemma 4.1 in \cite{van1993hellinger}]\label{lem:weighted-hellinger}
    For two density functions $g_1$ and $g_2$, define $g_u = u g_1 + (1-u) g_2$, then we have $$
    \frac{1}{4(1-u)}H^2(g_1\mid g_u) \leq H^2(g_1\mid g_2) \leq \frac{1}{(1-u)^2} H^2(g_1\mid g_u)
    $$
    holds for all $u\in(0,1)$
\end{lemma}
\begin{proof}
    For a detailed proof, see Lemma 4.1 in \cite{van1993hellinger}.
\end{proof}
\begin{lemma}[Lemma 5 in \citet{bennett2023source}]\label{lem:iter-tikh-bias}
    If $h_0$ is the minimum $L_2$-norm solution to the linear inverse problem and satisfies the $\beta$-source condition, then the solution to the $t$-th iterate of Tikhonov
regularization $h_{m,*}$, defined in Equation \eqref{eq:popu-iter}, with $h_{0,*}= 0$ , satisfies that
$$
\|h_{m,*} - h_0\|^2 \leq \|w_0\|^2 \alpha^{\min\{\beta, 2t\}}, \qquad\|\Tcal h_{m,*} - \Tcal h_0\|^2 \leq \|w_0\|^2 \alpha^{\min\{\beta+1, 2t\}}.
$$
\end{lemma}
\begin{proof}
    For a detailed proof, see Lemma 5 in \cite{bennett2023source}.
\end{proof}
The following lemma upper-bounds the bias introduced by Tikhonov regularization.
\begin{lemma}\label{lem:aux-1}
    For $$
    x^2 \leq c_1 + c_2 x^{\gamma_1} + c_3 x^{\gamma_2},
    $$
    where $c_1, c_2>0$, $0\leq\gamma\leq 1$, we have $x \leq 3 \max\big\{\sqrt{c_1}, c_2^{1/(2-\gamma_1)},c_3^{1/(2-\gamma_2)}\big\}$. 
\end{lemma}
\begin{proof}
   Since $x^2 - c_2 x^{\gamma_1} -c_3 x^{\gamma_2} -c_1$ is a convex function with negative intercept, we only need to prove that for $x_0 = 3 \max\big\{\sqrt{c_1}, c_2^{1/(2-\gamma_1)},c_3^{1/(2-\gamma_2)}\big\}$, we have $x_0^2 - c_2 x_0^{\gamma_1} -c_3 x_0^{\gamma_2} -c_1 \geq 0$. For simplicity, we consider $\sqrt{c_1} \geq \max\{ c_2^{1/(2-\gamma_1)},c_3^{1/(2-\gamma_2)}\}$, and  we have $$
   x_0^2 = 9c_1\geq c_1 + c_2\cdot3^{\gamma_1}  c_1^{\gamma_1/2}+c_3\cdot3^{\gamma_2}  c_1^{\gamma_2/2} = c_1+ c_2 x_0^{\gamma_1} + c_3 x_0^{\gamma_2},
   $$
   similarly we have the same result when $ c_2^{1/(2-\gamma_1)}\geq \max\{ \sqrt{c_1},c_3^{1/(2-\gamma_2)}\}$ or $c_3^{1/(2-\gamma_2)}\geq \max\{ \sqrt{c_1},c_2^{1/(2-\gamma_1)}\}$,
   and we conclude the proof.
\end{proof}
Next, we introduce the following lemma that gives a uniform convergence rate for function class $\Fcal$, which is adapted from \citet{wainwright2019high}.
\begin{lemma}[Theorem 14.20 in \cite{wainwright2019high}.]\label{lem:2cond}
    Suppose we have a $1$-uniformly bounded function class $\Fcal$ that is star-shaped around a population minimizer $f^*$. Let $\delta_n\geq \frac{c}{n}$ be the solution to the inequality $$
    \bar{R}_n(\delta;\Fcal^*)\leq\delta^2.
    $$
    Suppose the loss function $\Lcal_f$ is $L$-Lipschitz, then with probability at least $1 - c_1 \exp(-c_2 n\delta_{n,\Fcal}^2/b)$, either of the following events holds for all $f\in\Fcal$: \begin{itemize}
        \item[(1)] $\|f - f^*\|_2\leq \delta_n$;
        \item[(2)] $|\mathbb{P}_n (\Lcal_f - \Lcal_{f^*} ) -\mathbb{P} (\Lcal_f - \Lcal_{f^*} ) | \leq 10L\delta_n\| f - f^*\|_2$.
    \end{itemize}
\end{lemma}
The following lemma is a classical result for localization and uniform laws. 
\begin{theorem}[Theorem 14.1 of \cite{wainwright2019high}.]\label{thm:localization}
Given a star-shaped and $b$-uniformly bounded function class $\mathcal{F}$, let $\delta_n$ be any positive solution of the inequality
$$
\bar{\mathcal{R}}_n(\delta ; \mathcal{F}) \leq \frac{\delta^2}{b} .
$$

Then for any $t \geq \delta_n$, we have
$$
\left|\|f\|_n^2-\|f\|_2^2\right| \leq \frac{1}{2}\|f\|_2^2+\frac{t^2}{2} \quad \text { for all } f \in \mathcal{F}
$$
with probability at least $1-c_1 e^{-c_2 \frac{n \delta_n^2}{b^2}}$. If in addition $n \delta_n^2 \geq \frac{2}{c_2} \log \left(4 \log \left(1 / \delta_n\right)\right)$, then
$$
\left|\|f\|_n-\|f\|_2\right| \leq c_0 \delta_n \quad \text { for all } f \in \mathcal{F}
$$
with probability at least $1-c_1^{\prime} e^{-c_2^{\prime} \frac{n_0^2}{b^2}}$.
\end{theorem}
The next lemma enables us to upper-bound KL divergence by Hellinger distance. 
\begin{lemma}[Example 14.10 in \citet{wainwright2019high}. ]\label{lem:h-kl}
    For any two density function $g_1$ and $g_2$, we have $$
    H^2(g_1\mid g_2) \leq 2 \kl(g_1\mid g_2).
    $$
\end{lemma}

\section{Additional Experiment Details}\label{sec:exp-details}
We follow the data-generating process in \cite{kallus2021causal} and \cite{cui2020semiparametric} to generate multi-dimensional variables $U,S,W,Q,A$ with  $A\in\{0,1\}$ as follows: \begin{itemize}
    \item[1.]  $S'\sim \Ncal(0, 0.5 I_{d_S})$, where $I_d$ is a $d$-dimension identity matrix.
    \item[2.] $A|S'\sim \operatorname{Ber}(p(S'))$ where $$
    p(S')  = \frac{1}{1+ \exp(0.125- 0.125\textbf{1}_d^\top S')},
    $$
    where $\textbf{1}_d$ is all-one vector.
    \item[3.] Draw $W^{\prime}, Q^{\prime}, U$ from
$$
W^{\prime}, Q^{\prime}, U \mid A, S^{\prime} \sim \mathcal{N}\left(\left[\begin{array}{c}
\mu_0+\mu_a A+\mu_s S^{\prime} \\
\alpha_0+\alpha_a A+\alpha_s S^{\prime} \\
\kappa_0+\kappa_a A+\kappa_s S^{\prime}
\end{array}\right],\left[\begin{array}{c}
\sigma_w^2, \sigma_{w q}^2, \sigma_{w u}^2 \\
\sigma_{w q}^2, \sigma_q^2, \sigma_{q u}^2 \\
\sigma_{w u}^2, \sigma_{q u}^2, \sigma_u^2
\end{array}\right]\right) .
$$
Here we set the parameters above as $\mu_0=\alpha_0=\kappa_0=0.2 \mathbf{1}_d, \alpha_a=\kappa_a=\mu_s=\alpha_s=\kappa_s=\mathbb{I}_d$, $\sigma_q^2=\sigma_u^2=\sigma_w^2=0.1\left(\mathbb{I}_d+\mathbf{1}_d \mathbf{1}_d^{\top}\right), \sigma_{w u}^2=\sigma_{z u}^2=0.1 \mathbf{1}_d \mathbf{1}_d^{\top}$. Finally, we choose $\sigma_{w q}^2$ and $\mu_a$ to ensure that $W^{\prime} \perp\left(A^{\prime}, Q^{\prime}\right) \mid U, S^{\prime}$, which is a prerequisite of proximal causal inference \citep[Condition 4 in Assumption 1]{kallus2021causal}. To achieve this, note that
\begin{align}\label{eq:exp-1}
\mathbb{E}\left[W^{\prime} \mid U, S^{\prime}, A, Q^{\prime}\right]=\mu_0+\mu_a A+\mu_s S^{\prime}+\Sigma_{w(q, u)} \Sigma_{q, u}^{-1}\left[\begin{array}{c}
Q^{\prime}-\alpha_0-\alpha_a A-\alpha_s S^{\prime} \\
U-\kappa_0-\kappa_a A-\kappa_s S^{\prime}
\end{array}\right]
\end{align}
where
$$
\Sigma_{w(q, u)}=\left(\sigma_{w q}^2, \sigma_{w u}^2\right), \quad \Sigma_{q, u}=\left[\begin{array}{l}
\sigma_q^2, \sigma_{q u}^2 \\
\sigma_{q u}^2, \sigma_u^2
\end{array}\right] .
$$

We simply select $\sigma_{w q}^2$ and $\mu_a$ so that Equation \eqref{eq:exp-1} does not depend on $A$ and $Q^{\prime}$.
\item[4.] Draw $Y$ from
$$
Y \mid X^{\prime}, U, W^{\prime} \sim \mathcal{N}\left(A+\mathbf{1}_d^{\top} S^{\prime}+\mathbf{1}_d^{\top} U+\mathbf{1}_d^{\top} W^{\prime}, 1\right) .
$$
\item[5.] Set $W' = W'_{[0:d_W]}$. Observe $S = g(S')$, $Q = g(Q')$, $W = g(W')$, where $g(\cdot)$ is a reversible function that operates component-wise on each variable.  
\end{itemize}

\end{document}